\pdfoutput=1
\documentclass{article}

\usepackage[final]{corl_2026} 

\usepackage{times}

\usepackage[numbers]{natbib}
\usepackage{multicol}

\usepackage[utf8]{inputenc} 
\usepackage[T1]{fontenc}    
\usepackage{hyperref}       
\usepackage{url}            
\usepackage{array}
\usepackage{booktabs}       
\usepackage{longtable}
\usepackage{amsfonts}       
\usepackage{amsmath,amssymb,amsfonts}        
\usepackage{nicefrac}       
\usepackage{microtype}      
\usepackage{xcolor}         
\definecolor{darkgreen}{rgb}{0.0, 0.5, 0.0}
\usepackage{graphicx}       
\usepackage{subcaption}

\makeatletter
\newcommand{\plotcrop}[7][]{%
  \edef\plotcrop@vp{%
    \the\dimexpr #4bp*24/100\relax\space
    \the\dimexpr \numexpr#3-#7\relax bp*24/100\relax\space
    \the\dimexpr #6bp*24/100\relax\space
    \the\dimexpr \numexpr#3-#5\relax bp*24/100\relax}%
  \expandafter\plotcrop@do\expandafter{\plotcrop@vp}{#1}{#2}}
\newcommand{\plotcrop@do}[3]{\includegraphics[viewport=#1, clip, #2]{#3}}
\makeatother
\usepackage{amsthm}
\usepackage{algorithm}     
\usepackage{algorithmicx}
\usepackage{algpseudocode}
\usepackage{tasks}
\settasks{
  label         = {\arabic*)},  
  label-width   = 1.5em,
  item-indent   = 20pt,
  column-sep    = 3em,
  after-skip    = 1em          
}

\usepackage{float}
\usepackage{placeins}
\usepackage{wrapfig2}
\usepackage{appendix}
\usepackage{rotating} 
\usepackage{adjustbox}

\usepackage{makecell}
\usepackage{pifont}
\newcommand{\cmark}{\ding{51}}
\newcommand{\xmark}{\ding{55}}

\definecolor{darkcyan}{HTML}{0080ff} 
\definecolor{lavender}{HTML}{CC33FF} 
\definecolor{cyan_accent}{HTML}{33E6FF}
\definecolor{electric_green}{HTML}{33CC00}
\definecolor{mydarkblue}{rgb}{0,0.08,0.45}
\hypersetup{
    colorlinks=true,
    linkcolor=mydarkblue,
    citecolor=mydarkblue,
    urlcolor=mydarkblue,
    pdftitle={Diversity You Can Actually Measure: A Fast, Model-Free Diversity Metric for Robotics Datasets},
    pdfauthor={Sreevardhan Sirigiri, Nathan Samuel de Lara, Christopher Agia, Florian Shkurti, Fabio Ramos}
}
\makeatletter
\renewcommand{\@conferencelocation}{Austin, Texas, USA}
\makeatother

\theoremstyle{plain}
\newtheorem{definition}{Definition}
\newtheorem{theorem}{Theorem}
\newtheorem{lemma}{Lemma}

\makeatletter
\newcommand{\listofappendices}{%
  \section*{Appendix Contents}%
  \@starttoc{app}
}

\newcommand{\appsection}[1]{%
  \section{#1}%
  \addcontentsline{app}{section}{\protect\numberline{\thesection}#1}%
}

\makeatother

\title{Diversity You Can Actually Measure: A Fast, Model-Free Diversity Metric for Robotics Datasets}

\makeatletter
\newcommand{\blfootnote}[1]{%
  \begingroup
  \renewcommand{\@makefntext}[1]{\noindent##1}%
  \renewcommand{\thefootnote}{}%
  \let\@footnotetext\H@@footnotetext%
  \footnotetext{#1}%
  \endgroup
}
\makeatother
\author{%
  Sreevardhan Sirigiri$^{1}$
  \And
  Nathan Samuel de Lara$^{2}$
  \And
  Christopher Agia$^{3}$
  \And
  Florian Shkurti$^{2,4}$
  \And
  Fabio Ramos$^{1,5}$
}

\usepackage{paper-design}

\begin{document}
\maketitle
\blfootnote{$^{1}$The University of Sydney. \qquad $^{2}$University of Toronto. \qquad $^{3}$Stanford University. \qquad $^{4}$ Allen Institute for Artificial Intelligence. \qquad $^{5}$NVIDIA, USA.\\
Correspondence to \texttt{ssir4919@uni.sydney.edu.au}.}

\begin{abstract}
Robotics datasets for imitation learning typically consist of long-horizon trajectories of different lengths over states, actions, and high-dimensional observations (e.g.\ RGB video), making it non-trivial to quantify diversity in a way that respects the underlying trajectory structure and geometry. We extend Shannon and von Neumann entropy to this setting by defining signature-transform-based entropy on the Gram matrix of a signature kernel over demonstrations, yielding entropy and diversity metrics that operate directly on the demonstration dataset. Building on these, we study how dataset diversity affects generalization in robot IL and propose a simple, model-free way to curate diverse demonstrations. We introduce \textsc{FAKTUAL} (FAst trajectory Kernel enTropy cUration for imitation Learning), a curation algorithm that selects a subset maximizing entropy given a subset-size budget. \textsc{FAKTUAL} is fully model-free, requires no access to the IL policy or rollouts, and adds negligible overhead relative to policy training. We evaluate on image- and state-based RoboMimic and MetaWorld benchmarks and four real-world manipulation tasks. Across tasks and architectures, diversity-aware curation with \textsc{FAKTUAL} improves downstream success rates over random selection, while being substantially more compute-efficient than recent robot data curation methods. Our results suggest that demonstration-dataset entropy is a practical tool for understanding and improving dataset diversity in robot IL. We provide a reference implementation in this \href{https://github.com/S-sirigiri/FAKTUAL}{GitHub repository}.
\end{abstract}

\section{Introduction} \label{sec:intro}

Although many prominent advances in deep learning have been driven by architectural innovations and dataset scale, the diversity of training data is also a crucial determinant of model performance. Dataset diversity is desirable across many areas of machine learning---dataset curation, generative modeling, reinforcement learning, active learning, and decoding algorithms~\cite{vendi_score}---and robotics, in particular imitation learning (IL), is no exception~\cite{o2024open, black2410pi0, intelligence2025pi05visionlanguageactionmodelopenworld, gao2024efficientdatacollectionrobotic, shi2025diversityneedscalablerobotic}. Recent work on large and (anecdotally) diverse datasets has been a central driver of scaling in robot imitation learning~\cite{o2024open, khazatsky2024droid, rt12022arxiv, rt22023arxiv, octo_2023, pmlr-v270-kim25c, black2410pi0, intelligence2025pi05visionlanguageactionmodelopenworld}.

Despite growing interest, there is no widely accepted measure of diversity for robot datasets, and quantifying it in a principled way is non-trivial for two reasons: (1) demonstrations are long-horizon trajectories of varying lengths spanning multiple modalities (states, actions, high-dimensional sensory observations such as video); and (2) a diversity metric requires a notion of similarity between demonstrations that goes beyond pointwise distance and respects the underlying trajectory structure and geometry.

In this work, we leverage the signature kernel to define such a notion of diversity. The signature transform and its associated signature kernel~\cite{lee2023signaturekernel, primerSK} provide a principled way to represent and compare paths of variable length while preserving their geometric and sequential structure, inducing a similarity measure that respects trajectory geometry rather than relying on pointwise distances. Prior work has demonstrated that signature-kernel-based methods discriminate between robotic paths in trajectory optimization, path planning, and ergodic search~\cite{PathSigtrajOpt, sirigiri2025diversifying}.

\begin{figure}[htbp]
    \centering
    \PaperFrame{PaperTeal}{\includegraphics[width=\linewidth,trim=55bp 195bp 50bp 145bp,clip]{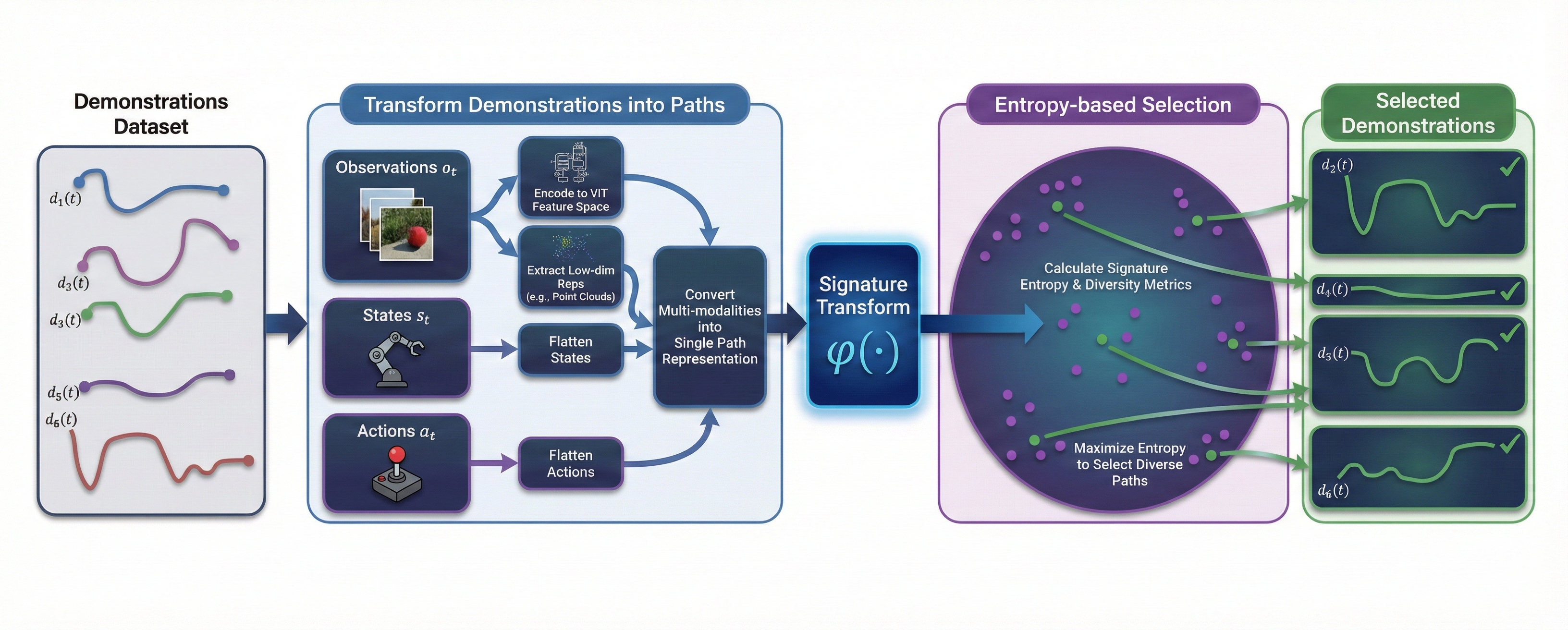}}
    \caption{A graphical depiction of \textsc{FAKTUAL}. For each demonstration we embed the observations (RGB images) into a ViT feature space, or alternatively extract low-dimensional representations of the objects of interest (e.g.\ point clouds). We then flatten any additional modalities (states, actions) if required, so each demonstration is represented as a path/trajectory in space (details in Appendix~\ref{App:demo_to_paths}). Using these path representations, we compute signature-based entropy/diversity metrics and select a subset of the dataset that maximizes entropy or diversity.}
    \label{fig:overview}
\end{figure}

Beyond spanning diverse behaviors and states, a dataset must provide sufficient \emph{coverage} around the various behaviors and states, i.e. enough local ``\emph{density}'' of demonstrations for a reliable IL policy, rather than just isolated, possibly outlying, trajectories. Demonstration \emph{quality} is likewise essential: a dataset can hinder learning if large portions are suboptimal or inconsistent~\cite{robomimic}. Quality, diversity, and density are thus complementary properties that jointly determine downstream policy performance.

Curating datasets using principled diversity metrics as a first-class signal remains largely unexplored, and existing (quality-based) pipelines often do not explicitly preserve diversity during selection. When dataset quality is already high, the importance of diversity and balanced coverage becomes even more pronounced~\cite{parekh2026balancedbehaviorcloningimbalanced}. Recent work has devoted considerable attention to curation for improving the quality of robot IL datasets~\cite{kuhar2023learning, DemInf, chen2025curating, demo_score, cupid}; a successful approach retains only demonstrations judged beneficial via a heuristic, \emph{task-agnostic} notion of quality, producing a smaller offline-curated dataset~\cite{DemInf}. A limitation of several such methods (e.g.~\citep{DemInf, demo_score, chen2025curating, cupid}) is that they typically rely on training an auxiliary model to compute heuristic scores, often incurring upfront costs comparable to training the target policy itself; due to this computational intensity, they are often best applied to smaller, task-specific fine-tuning datasets (e.g.\ post-training) rather than large-scale multi-task pre-training datasets. If diversity can instead be computed cheaply and reliably, it provides a practical dataset-level signal for curation without an additional learned scoring model, and can be combined with quality-based curation to obtain datasets with positive quality and diversity properties. While diversity is not itself a direct measure of task success, we show later that demonstration diversity correlates well with policy success rate, supporting its use as a practical curation criterion.

\PaperBegin{contributions}{wash}{PaperPlum}{0pt}{0pt}
\textbf{Contributions:} (1) We compute Shannon and von Neumann entropy on the eigenvalues of the signature kernel, providing a computationally cheap, principled way to measure entropy/diversity of robotics datasets that explicitly incorporates states, actions, videos, and other modalities. (2) We propose \textsc{\textbf{FAKTUAL}} (\textsc{FA}st trajectory \textsc{K}ernel en\textsc{T}ropy c\textsc{U}ration for imitation \textsc{L}earning), a curation strategy grounded in this entropy/diversity notion that selects a subset of the full dataset maximizing entropy/diversity, removing redundant demonstrations and yielding more uniform coverage of trajectory modes (states and behaviors). Our strategy is completely model-free, requires no information about the IL policy or its rollouts, and depends solely on the dataset itself. We show empirically that diversity-aware curation, though a heuristic, improves policy performance.
\par\PaperEnd{contributions}

\section{Background: Signature Transform and Signature Kernel} \label{sec:sig_trans_and_sig_kern_main}
Informally, the signature transform of a trajectory is analogous to a Fourier series: it represents a path of variable length as a countably infinite sequence of features capturing its geometric and temporal structure.

\PaperBegin{background1}{wash}{PaperTeal}{0pt}{0pt}
\begin{definition}\textbf{(Informal) Signature transform of a path~\citep{sigPDE}.} \label{sigT} Let \(x: I \to \mathbb{R}^d\) be a continuous path on an interval \(I \subset \mathbb{R}\). For any subinterval \([s,t]\subset I\), the \emph{signature} of \(x\) over \([s,t]\) is the sequence
\[
\varphi(x)_{s,t} = \bigl(1,\, \int_{s<u_1<t}dx(u_1),\, \ldots,\, \int_{s<u_1<\cdots<u_k<t}dx(u_1)\otimes\cdots\otimes dx(u_k),\, \ldots\bigr),
\]
 where \(\otimes\) denotes the (classical) tensor product. This family of iterated integrals encodes salient information about the path and is invariant to reparameterization (i.e.\ \(\varphi(x)=\varphi(x\circ\theta)\) for any reparameterization \(\theta\)).
\end{definition}
\PaperEnd{background1}

The path signature is a canonical linear embedding of a multivariate path into an (infinite) series of iterated integrals. It is injective on the class of non-tree-like paths (in practice, non-tree-likeness is typically enforced by appending a monotone ``time'' coordinate). The signature also obeys a \emph{time-reversal duality}: \(\varphi(x)_{a,b}\otimes\varphi(\overleftarrow{x})_{a,b}=1\), where \(1\) is the unit element and \(\overleftarrow{x}(t)=x(a+b-t)\).

\PaperBegin{background2}{wash}{PaperTeal}{0pt}{0pt}
\begin{definition}\textbf{(Informal) Signature kernel.} Let \(I=[u,u']\), \(J=[v,v']\) be intervals and \(x\in\mathcal{C}^1(I,\mathbb{R}^d)\), \(y\in\mathcal{C}^1(J,\mathbb{R}^d)\). The \emph{signature kernel} \(k^{sig}(x,y): I\times J\to\mathbb{R}\) is
\[
k^{sig}_{s,t}(x,y)=\langle \varphi(x)_s,\, \varphi(y)_t\rangle,
\]
 where \(\varphi(x)_s,\varphi(y)_t\) are the signatures of \(x,y\) over \([u,s],[v,t]\).
\end{definition}
\PaperEnd{background2}

In practice, the signature kernel is well approximated by the \emph{truncated signature transform} at level \(L\in\mathbb{N}\) (Appendix~\ref{app:rfsf}), \(\varphi_{s,t}^{L}(x)=\bigl(1,\,\int_{s<u_1<t}dx(u_1),\,\ldots,\,\int_{s<u_1<\cdots<u_L<t}dx(u_1)\otimes\cdots\otimes dx(u_L)\bigr)\). \citet{DPsig} and \citet{rffsig} provide efficient procedures for computing the truncated signature transform and its kernel via dynamic programming and random Fourier features (Appendix~\ref{app:rfsf}), respectively. The untruncated signature transform and kernel can be computed via a \emph{kernel trick} (Appendix~\ref{app:sigkern}) introduced by \citet{sigPDE}.

\section{Problem Statement} \label{sec:prob_statement}
We aim to characterize the relationship between demonstration dataset diversity and model performance in closed-loop robotic imitation learning, and to leverage this relationship to motivate a diversity-based pruning method. We consider a finite-horizon decision process \(\langle \mathcal{S}, \mathcal{A}, \mathcal{T}, R, \rho_0, \mathcal{O}, \cdots \rangle\) with state space \(\mathcal{S}\), action space \(\mathcal{A}\), transition model \(\mathcal{T}\), reward \(R\), initial state distribution \(\rho_0\), observation space \(\mathcal{O}\), and horizon \(T\). We train a policy \(\pi_\theta\) by minimizing a behavior cloning (BC) objective, \(\theta = \arg\min_{\theta'} \mathcal{L}(\theta'; \mathcal{D})\) with
\[
\mathcal{L}(\theta'; \mathcal{D}) := \tfrac{1}{|\mathcal{D}|}\sum_{d_i \in \mathcal{D}} \tfrac{1}{T_i}\sum_{(s,a)\in d_i} \ell(s,a;\pi_{\theta'}),
\]
 using a demonstration dataset $\mathcal{D}=\{d_1,\ldots,d_n\}$. Each demonstration $d_i = ((s_i(0), a_i(0), o_i(0)),\allowbreak\ldots,(s_i(T_i), a_i(T_i), o_i(T_i)))$ is a state-action-observation trajectory in which the robot successfully completes the task. We treat a trajectory \(\tau = (s(0), a(0), \ldots, s(T))\) as either a \emph{success} (binary return \(R_i(t)=1\)) or \emph{failure} (\(R_i(t)=-1\)) for some \(t\in T\).

In IL we train \(\pi_\theta\) to match the distribution of successful behaviors in \(\mathcal{D}\) (rather than directly maximizing expected return \(J(\pi_\theta) := \mathbb{E}_{p(\tau\mid\pi_\theta)}[R(\tau)]\)). As a consequence, policy performance is tightly coupled to demonstration data composition---not solely to validation losses, model capacity, or bias--variance trade-offs~\citep{cupid}---making systematic improvement challenging. Recent work shows that merely scaling demonstration collection can produce datasets with redundancy and behaviors that degrade policy performance, even when \(R(d_i)=1\) for all \(d_i\in\mathcal{D}\)~\cite{belkhale2023data}.

We focus on finding a measure of dataset diversity/information content \(f(D)\in\mathbb{R}_{\geq 0}\) such that \(f(\cdot)\) correlates positively with the success rate of a policy trained on \(D\subseteq\mathcal{D}\), \(\mathcal{J}(\pi_D)\), and on using \(f\) to find subsets \(D\subseteq\mathcal{D}\) with \(f(D)\approx f(\mathcal{D})\), hence \(\mathcal{J}(\pi_D)\approx\mathcal{J}(\pi_{\mathcal{D}})\).

\section{Measuring Diversity of Datasets in Robotics} \label{sec:dataset_diversity}
The usual definition of Shannon entropy (Eq.~\eqref{eq:shannon_entropy}; Appendix~\ref{sec:background}) requires a probability distribution (or a tractable way to compute it), which is not feasible in many robotics contexts. We aim for a diversity metric that depends only on the samples themselves, which requires a similarity function over samples; the signature kernel of Section~\ref{sec:sig_trans_and_sig_kern_main} serves this purpose. Reparameterization invariance effectively filters out the infinite-dimensional group of symmetries, and the linear functionals on the signature form an algebra under pointwise multiplication that separates points~\citep{sigisDense}. By the \textit{Stone--Weierstrass} theorem~\cite{Rudin1976PMA}, for any compact set \(\mathbf{C}\) of continuous paths with bounded variation, these functionals are dense in the space of continuous real-valued functions on \(\mathbf{C}\)~\citep{sigisDense2}, motivating the signature as a powerful feature map for representing demonstrations once they are encoded as paths.

We convert each demonstration into a (multivariate) path by (i) embedding raw visual observations at each timestep into a fixed-dimensional vector (e.g., a ViT embedding or a low-dimensional representation), and (ii) concatenating with the per-timestep robot state and action to form a trajectory \(X_i(t)\). Alternatively, we treat each modality (states, actions, observations, etc.) as its own path and combine them via a convex combination of the corresponding signature kernels (see Appendix~\ref{App:demo_to_paths}). We can now construct an entropy that relies only on the evaluated samples.

\PaperBegin{methodology1}{wash}{PaperTeal}{0pt}{0pt}
\begin{definition}[Signature Shannon Entropy] \label{def:sig_shannon_entropy}
    Let \(k^{sig}\) be the signature kernel, \(d_1, \ldots, d_n \in \mathcal{D}\) a set of demos, and \(K^{sig} \in \mathbb{R}^{n \times n}\) the normalized kernel matrix with entries \(K_{i,j}^{sig}= k^{sig}(d_i, d_j)\) and \(K_{i,i}^{sig}=1\). The \emph{signature Shannon entropy} is
\[
H_\text{shannon}^{sig}(d_1, \ldots, d_n) = - \sum_{\lambda_i \in \Lambda} \lambda_i \ln \lambda_i,
\]

    where \(\Lambda\) is the set of eigenvalues of \(K^{sig}/n\).
\end{definition}
\PaperEnd{methodology1}

The eigenvalues of \(K/n\) are nonnegative (\(k^{sig}\) is PSD) and sum to one (each diagonal entry of \(K/n\) equals \(1/n\)), so Shannon entropy is well-defined.

\PaperBegin{methodology2}{wash}{PaperTeal}{0pt}{0pt}
\begin{definition}[Signature von Neumann Entropy] \label{def:sig_neumann_entropy}
    With the notation of Definition~\ref{def:sig_shannon_entropy}, the \emph{signature von Neumann entropy} is
\[
H^{sig}_\text{von neumann}(d_1, \ldots, d_n) = -\mathrm{Tr}\!\left(\frac{K^{sig}}{n} \ln \frac{K^{sig}}{n}\right).
\]

\end{definition}
\PaperEnd{methodology2}

Remarkably, \(H^{sig}_\text{shannon} = H^{sig}_\text{von neumann}\) (Appendix~\ref{app:proofs};~\citep{vendi_score}); exponentiating this entropy yields the effective rank of the kernel matrix~\citep{effective_rank}, defining the Vendi score~\citep{vendi_score}.

The entropy of a dataset reflects how uniformly random and well-distributed the samples are. Another useful metric is the volume spanned by the demonstrations, measured via \(\Delta(d_1, \ldots, d_n):=\det(K^{sig})\), with \(K^{sig}_{i,j} = k^{sig}(d_i,d_j)\).\footnote{Technically, \(\Delta\) is the volume of the parallelotope spanned by the feature embeddings in the RKHS induced by \(k\).} The kernel \(k(d_i,d_j)\) is a similarity score between demonstrations: for a normalized kernel (\(k(d_i,d_j)\le 1\), \(k(d_i,d_i)=1\)), \(k(d_i,d_j)\to 1\) as \(d_i \to d_j\) and \(\to 0\) as \(\|d_i-d_j\|\to\infty\). Given trajectories, we form the Gram matrix \(K\in\mathbb{R}^{n\times n}\) with \(K_{ij}=k(d_i,d_j)\), \(K_{ii}=1\), and evaluate its determinant: \(\det(K)\to 0\) if \(d_i\to d_j\) for any \(i\neq j\), and \(\det(K)\to 1\) as \(\|d_i-d_j\|\to\infty\) for all \(i\neq j\), yielding another principled measure of diversity.

\section{Data Curation with Signature Diversity Metrics} \label{sec:dataset_curation}

\noindent We first introduce two definitions. Let \(H := H^{\text{sig}}_{\text{shannon}} = H^{\text{sig}}_{\text{von neumann}}\).

\PaperBegin{methodology3}{wash}{PaperTeal}{0pt}{0pt}
\begin{definition}[Signature-entropy/determinant--maximizing \(m\)-subsets] \label{def:max_entropy_set}\label{def:max_det_set}
With the notation of Definition~\ref{def:sig_shannon_entropy}, the \emph{Signature-entropy--maximizing \(m\)-subset} \(\hat D^{\text{entropy}}_m\) and \emph{Signature-determinant--maximizing \(m\)-subset} \(\hat D^{\text{det}}_m\) satisfy, respectively, \(H(\hat D^{\text{entropy}}_m)\geq H(D_m)\) and \(\Delta(\hat D^{\text{det}}_m)\geq \Delta(D_m)\) for all \(D_m\subseteq\mathcal{D}\) with \(|D_m|=m\).\PaperEndHere{methodology3}
\end{definition}
\PaperEnd{methodology3}

Finding either the signature-determinant--maximizing or signature-entropy--maximizing \(m\)-subset reduces to the NP-Hard maximum coverage problem (a.k.a. max-\(k\)-cover) \citep{np_hardness, maximum_coverage}, so we obtain approximate solutions in polynomial time. Typically, algorithms employ a greedy strategy, iteratively adding demonstrations that maximally increase the entropy or determinant; see Appendix~\ref{app:algos}.

\textbf{Data Curation Strategy}. To curate a dataset of size \(m\), we select the (approximate) signature-entropy--maximizing \(m\)-subset of \(\mathcal{D}\). This is justified because, in practice, the signature entropy \(H\) exhibits diminishing returns as data points are added: beyond a certain subset size, more points only marginally increase the entropy, which approaches saturation (Figure~\ref{fig:entropy_vs_num_demos}). Hence, a subset curated by maximizing entropy can achieve nearly the same ``information content'' as the full dataset while using far fewer demonstrations, while encouraging uniform coverage/density.

\begin{figure}[htbp]
  \centering
  \begin{minipage}[c]{0.60\linewidth}
    \PaperFrame{PaperTeal}{\includegraphics[width=\linewidth]{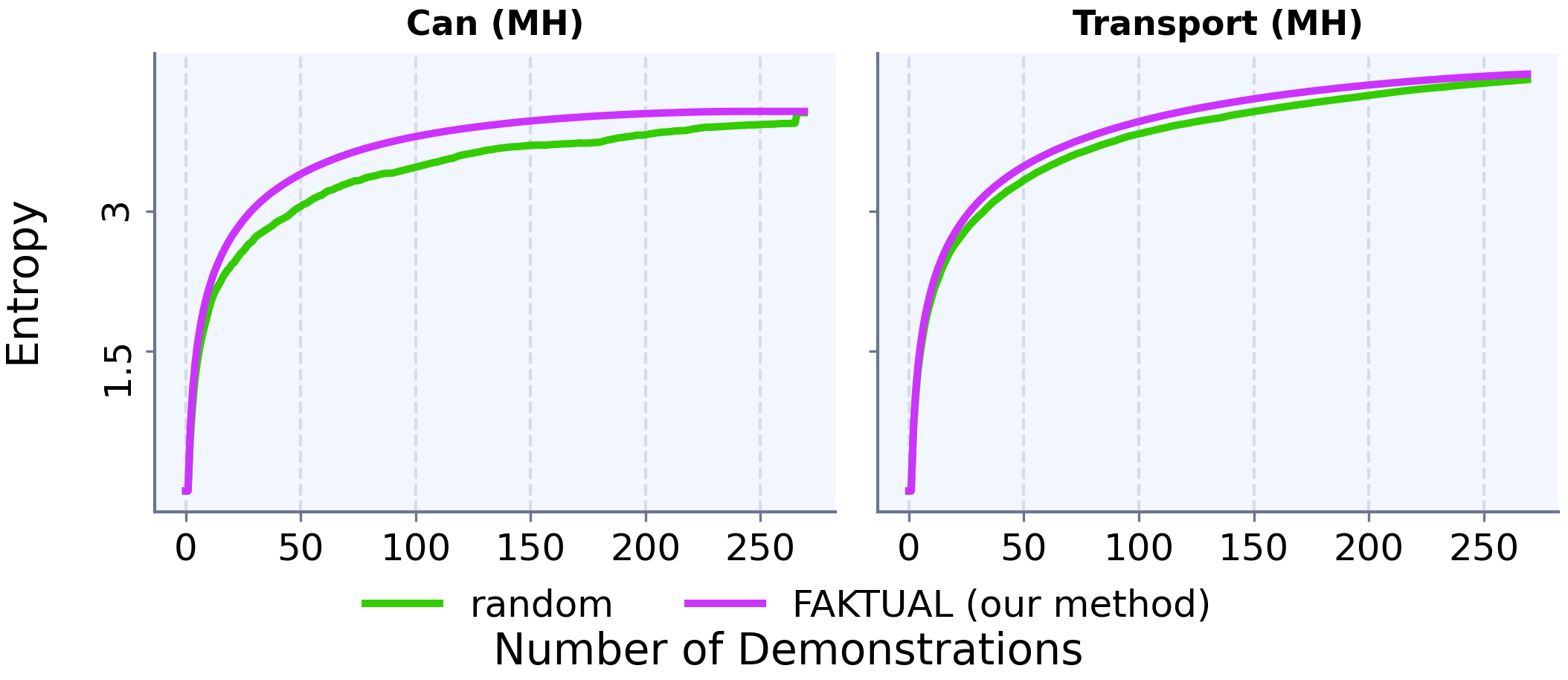}}
  \end{minipage}\hfill
  \begin{minipage}[c]{0.37\linewidth}
  \captionsetup{justification=raggedright,singlelinecheck=false}
  \caption{Entropy vs.\ number of demonstrations. For Can, entropy saturates quickly and approaches an asymptote. For Transport it does not saturate over the range shown, indicating Transport is more diverse (at least as discriminated by the signature entropy) and each demo contributes almost equally to entropy, yielding weaker separation between random and \textsc{FAKTUAL}. 
  }
  \label{fig:entropy_vs_num_demos}
  \end{minipage}
\end{figure}

Maximizing entropy promotes uniform randomness but may overlook extreme data points, which can be beneficial to include (unless they are very poor and should have been filtered out). An objective maximizing the volume of the subset, e.g.\ \(\det K\), is useful here. Concretely, to curate a dataset of size \(m\): (1) fix a proportion \(p\in[0,1]\) (in practice \(p\in[0.5,0.9]\) works well); (2) select the (approximate) signature-entropy--maximizing subset \(\hat D^{entropy}_{mp}\subseteq\mathcal{D}\) of size \(mp\); (3) from the remaining samples, select the (approximate) signature-determinant--maximizing subset \(\hat D^{det}_{m(1-p)}\subseteq\mathcal{D}\setminus\hat D^{entropy}_{mp}\) of size \(m(1-p)\); (4) return \(\hat D^{entropy}_{mp}\cup\hat D^{det}_{m(1-p)}\). Algorithm details are in Appendix~\ref{app:algos}.

\section{Experiments and Results} \label{sec:results}
We present results across simulated and real experiments (setup details in Appendix~\ref{app:exps_deatails}), addressing: (1) success rate vs.\ curated dataset size \(m\) across tasks; (2) success rate of our method vs.\ baseline curation strategies; (3) the effect of diversity on success rate.

\subsection{Simulated Experiments}

We evaluate on RoboMimic~\citep{robomimic} and MetaWorld~\citep{metaworld} tasks: RoboMimic \emph{Can} and \emph{Square} (each with proficient-human (PH) and multi-human (MH) datasets), RoboMimic \emph{Transport} (MH), and MetaWorld \emph{Stick Push} (stick-push-v3), \emph{Door Open} (door-open-v3), and \emph{Shelf Place} (shelf-place-v3). We use behavioral cloning~\citep{bc} with an RNN~\citep{rnn1, rnn2, rnn3} and BC with a diffusion action head~\citep{diffusion_policy} for Transport and Square.\footnote{We use BC with RNN over diffusion policy for Can since the full-data policy already achieves 100\% success.} For MetaWorld tasks we fine-tune SmolVLA~\citep{smolvla}. Results are presented in Figures~\ref{fig:robomimic_results} and~\ref{fig:metaworld_results}. For all RoboMimic tasks here we use the image-based datasets provided by RoboMimic.

We further analyze the effect of diversity on success rate and Figure~\ref{fig:entropy_success} shows a strong positive association between signature entropy and downstream success: across tasks, higher-entropy sets consistently achieve higher success rates (Pearson $r=0.74$--$1.00$ for both random selection and \textsc{FAKTUAL}).

\subsection{Real-world Experiments}
We evaluate on four real-world manipulation tasks using an SO-ARM101 6-DoF robotic arm kit~\citep{so101_github}: \emph{Drawer open} (opening a closed drawer); \emph{Mug drag} (grasping a hook and dragging a mug to a goal using the hook as an intermediate tool); \emph{Pick-and-place marker} (grasping a marker and placing it into a pen holder); and \emph{Organize objects} (a conditional task: if a marker is on the table, place it on a book; else if a charging brick is on the table, open the drawer, place the brick inside, and close the drawer). We again use SmolVLA~\citep{smolvla}. Results are in Figure~\ref{fig:real_results}.

\begin{figure}[!htbp]
\centering
\scalebox{.90}{\PaperFrame{PaperPlum}{\begin{minipage}{\linewidth}
\centering
\parbox[c]{.3333\linewidth}{\centering\plotcrop[scale=.43]{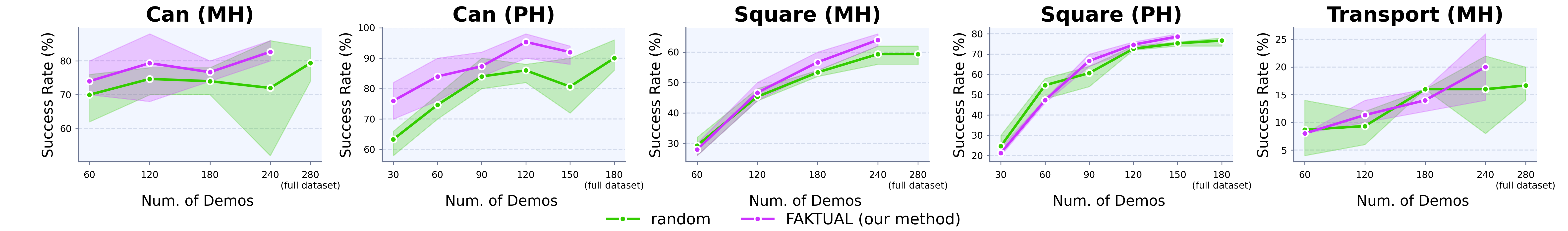}{930}{158}{15}{1425}{838}}%
\parbox[c]{.3333\linewidth}{\centering\plotcrop[scale=.43]{Plots/robomimic_results.png}{930}{1433}{15}{2594}{838}}%
\parbox[c]{.3333\linewidth}{\centering\plotcrop[scale=.43]{Plots/robomimic_results.png}{930}{2598}{15}{3815}{838}}\\[3pt]
\parbox[c]{.3333\linewidth}{\centering\plotcrop[scale=.43]{Plots/robomimic_results.png}{930}{3819}{15}{5036}{838}}%
\parbox[c]{.3333\linewidth}{\centering\plotcrop[scale=.43]{Plots/robomimic_results.png}{930}{5040}{15}{6256}{838}}%
\parbox[c]{.3333\linewidth}{\centering
  \begin{tabular}{@{}l@{}}
    \plotcrop[scale=.43]{Plots/robomimic_results.png}{930}{2425}{849}{2860}{918}\\
    \plotcrop[scale=.43]{Plots/robomimic_results.png}{930}{2967}{849}{3865}{918}
  \end{tabular}}
\end{minipage}}}
\caption{RoboMimic curation results. Success rates are computed over 50 rollouts for each of 20 checkpoints; we report the maximum across checkpoints, averaged over three random seeds, with error bars indicating min/max values. On Can, \textsc{FAKTUAL} consistently outperforms random even at the lowest demo counts, with a particularly large gap for PH where performance improves sharply with very few demonstrations. For Square, curves are largely indistinguishable at low counts; at higher counts MH intervals still slightly overlap while PH shows a small but consistent advantage for \textsc{FAKTUAL}. This PH advantage is plausibly because PH comprises proficient-human, high-quality demos and \textsc{FAKTUAL} is a diversity-based (not quality-based) strategy (consistent with the uniform coverage/density discussion in Section~\ref{sec:intro} when data quality is high). Transport is more challenging: at small budgets \textsc{FAKTUAL} is comparable to or marginally below random, but trends toward an advantage as demonstrations increase. \textsc{FAKTUAL} may eventually outperform random as more demonstrations are added and entropy begins to saturate (Figure~\ref{fig:entropy_vs_num_demos}).}
\label{fig:robomimic_results}

\bigskip

\scalebox{.90}{\PaperFrame{PaperTeal}{\begin{minipage}{\linewidth}
\centering
\parbox[c]{.3333\linewidth}{\centering\plotcrop[scale=.47]{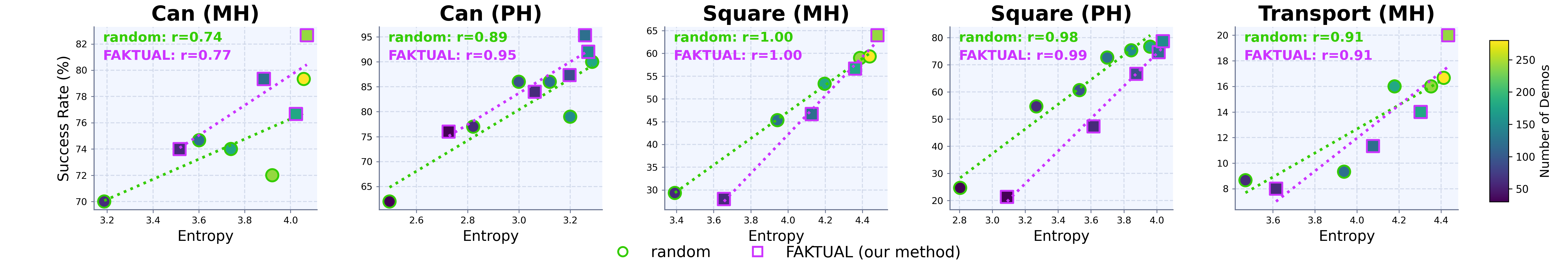}{1080}{222}{10}{1285}{981}}%
\parbox[c]{.3333\linewidth}{\centering\plotcrop[scale=.47]{Plots/entropy_vs_success_rate.png}{1080}{1445}{10}{2430}{981}}%
\parbox[c]{.3333\linewidth}{\centering\plotcrop[scale=.47]{Plots/entropy_vs_success_rate.png}{1080}{2591}{10}{3576}{981}}\\[3pt]
\parbox[c]{.3333\linewidth}{\centering\plotcrop[scale=.47]{Plots/entropy_vs_success_rate.png}{1080}{222}{10}{298}{981}%
  \plotcrop[scale=.47]{Plots/entropy_vs_success_rate.png}{1080}{3736}{10}{4722}{981}}%
\parbox[c]{.3333\linewidth}{\centering\plotcrop[scale=.47]{Plots/entropy_vs_success_rate.png}{1080}{4883}{10}{5868}{981}}%
\parbox[c]{.3333\linewidth}{\centering
  \plotcrop[scale=.47]{Plots/entropy_vs_success_rate.png}{1080}{5975}{150}{6230}{822}\\[4pt]
  \begin{tabular}{@{}l@{}}
    \plotcrop[scale=.47]{Plots/entropy_vs_success_rate.png}{1080}{2470}{983}{2860}{1046}\\
    \plotcrop[scale=.47]{Plots/entropy_vs_success_rate.png}{1080}{3011}{983}{3865}{1046}
  \end{tabular}}
\end{minipage}}}
\caption{Signature entropy correlates with success. Success rate vs.\ signature entropy across RoboMimic tasks (Can, Square, Transport) and dataset variants (MH/PH). Points are curated subsets colored by demonstration count; dotted lines are linear fits for random (green circles) and \textsc{FAKTUAL} (magenta squares). Pearson $r$ per panel indicates a strong positive association between entropy and downstream success.}
\label{fig:entropy_success}
\end{figure}

\begin{figure}[t]
\centering
\PaperFrame{PaperTeal}{\includegraphics[width=1.0\linewidth]{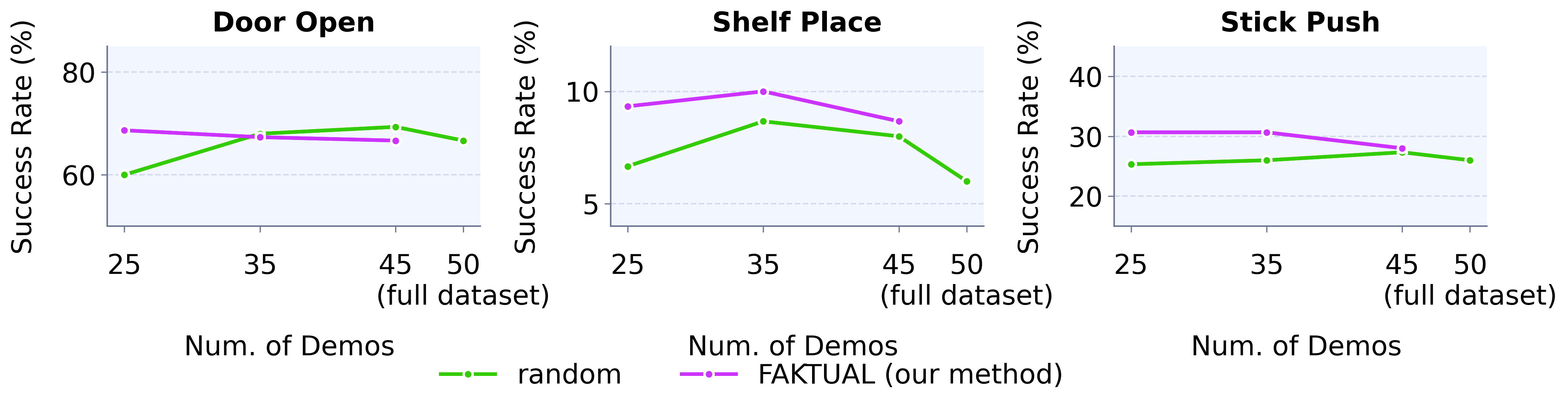}}
\caption{MetaWorld curation results. Success rates over 50 rollouts per checkpoint, max across 20 checkpoints. MetaWorld~\cite{metaworld} reports substantially different success rates across seeds; we plot only the mean over 3 seeds, with full per-seed results in Table~\ref{tab:full_results_metaworld}.}
\label{fig:metaworld_results}

\medskip

\PaperFrame{PaperPlum}{\parbox[c]{.585\linewidth}{\centering
  \plotcrop[scale=.43]{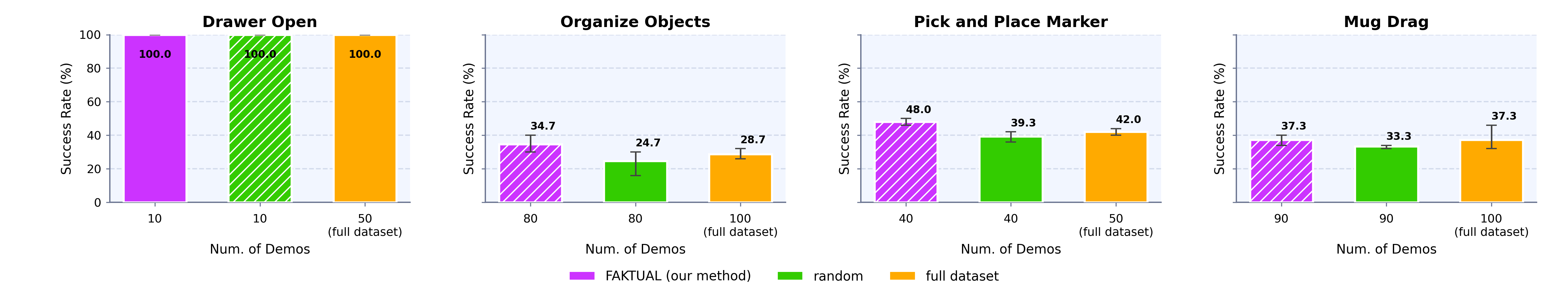}{930}{186}{42}{1328}{826}\hspace{4pt}%
  \plotcrop[scale=.43]{Plots/real_results.png}{930}{1480}{42}{2535}{826}\\[2pt]
  \plotcrop[scale=.43]{Plots/real_results.png}{930}{186}{42}{340}{826}%
  \plotcrop[scale=.43]{Plots/real_results.png}{930}{2754}{42}{3742}{826}\hspace{4pt}%
  \plotcrop[scale=.43]{Plots/real_results.png}{930}{3893}{42}{4949}{826}}}%
\hfill
\parbox[c]{.395\linewidth}{%
  \plotcrop[scale=.63]{Plots/real_results.png}{930}{1825}{864}{2418}{911}\hspace{3pt}%
  \plotcrop[scale=.63]{Plots/real_results.png}{930}{2494}{864}{2778}{911}\hspace{3pt}%
  \plotcrop[scale=.63]{Plots/real_results.png}{930}{2855}{864}{3215}{911}\par
  \captionsetup{justification=raggedright}%
  \caption{Real-task curation results. Success rates over 50 rollouts per checkpoint, max across 20 checkpoints, averaged over three seeds; error bars indicate min/max. Hashed denotes the best-performing method (max success rate); ties go to the simplest method. Three of four real-world tasks show the curated subset outperforming the full dataset, indicating that additional demonstrations can be redundant and add noise rather than help.}
\label{fig:real_results}}
\end{figure}


\subsection{Baselines}

\begin{wraptable}{r}{0.5\textwidth}
\captionsetup{hypcap=false}

\centering
\footnotesize
\setlength{\tabcolsep}{3pt}
\caption{Data curation strategy characteristics}
\label{tab:curation_methods}
\begin{tabular}{@{}>{\columncolor{white}[0pt][\tabcolsep]}lcc>{\columncolor{white}[\tabcolsep][0pt]}c@{}}
\toprule
\PaperTableHead
Method & \makecell{Aux.\\model?} & \makecell{Model/Policy\\-free?} & \makecell{Online\\rollouts?} \\
\midrule
CUPID              & \xmark & \xmark & \cmark \\
Demo-SCORE         & \cmark & \xmark & \cmark \\
DemInf             & \cmark & \cmark & \xmark \\
Success Similarity & \xmark & \xmark & \cmark \\
\PaperMethodRow
\textsc{FAKTUAL}   & \xmark & \cmark & \xmark \\
\bottomrule
\end{tabular}

\end{wraptable}

We compare our method against the following baselines: (1) DemInf~\citep{DemInf}, which estimates the relative quality of individual demonstrations in terms of action diversity and predictability;\footnote{Note that DemInf can only perform Demo-filtering, while \textsc{FAKTUAL} can only perform Demo-Selection, as trying to perform filtering with \textsc{FAKTUAL} would correspond to finding a set of demos with the lowest diversity and pruning it out, which is ill defined as it would just prune out a "local cluster" of demos.} (2) Cupid~\citep{cupid}, based on a novel influence-function-theoretic formulation for IL policies; (3) Demo-SCORE~\citep{demo_score}, which self-curates from online robot experience by training/cross-validating a classifier to discern successful from unsuccessful policy rollouts and using it to filter demonstrations; (4) Oracle, approximating an upper bound on achievable performance using privileged access to ground-truth demonstration labels (e.g., quality indicators, strategy robustness, or other relevant properties); (5) Success Similarity, a custom baseline from \citet{cupid} ranking demonstrations by average state similarity to successful rollouts; and (6) Random selection. Table~\ref{tab:curation_methods} summarizes of key design characteristics of robot demonstration
data curation strategies, categorized by auxiliary model (i.e., an
additional learned model beyond the imitation policy itself) requirements,
dependence on a learned policy, and use of online environment rollouts. To make fair comparisons, we follow the exact setup and implementation from \citet{cupid}. Results are in Figure~\ref{fig:baseline_results}.\footnote{In this setting we use RoboMimic's low-dimensional image representations rather than raw RGB, so success rates may differ from previous experiments; our method does not rely on ViT embeddings here (see Appendix~\ref{App:demo_to_paths}).}

\begin{figure}[!htbp]
    \centering
    \PaperFrame{PaperTeal}{\includegraphics[width=1.0\textwidth]{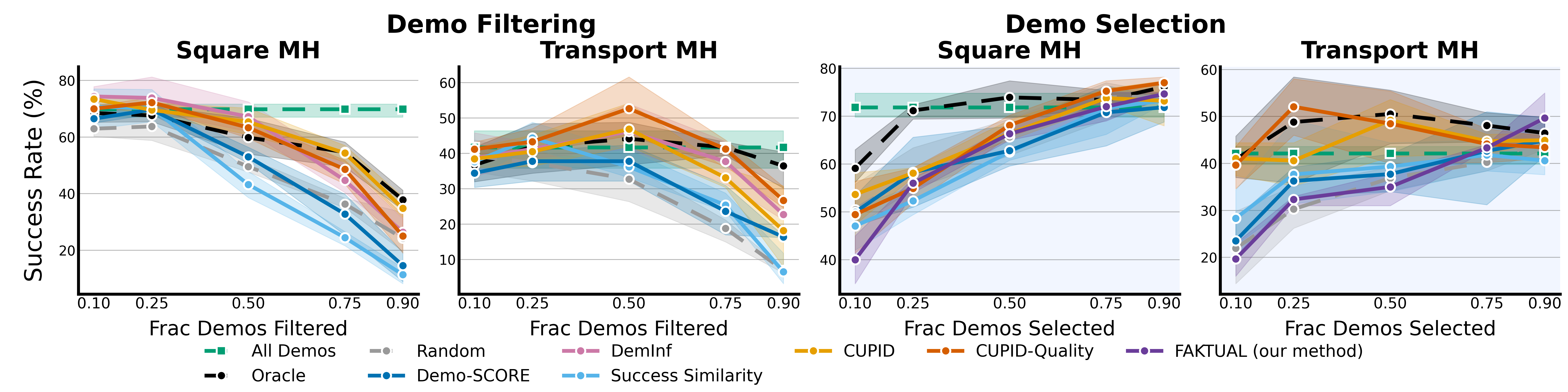}}
    \caption{RoboMimic curation results with baselines. Success rates over 50 rollouts per checkpoint, max across 20 checkpoints, averaged over three seeds; error bars indicate min/max. Relative to random pruning, \textsc{FAKTUAL} is generally better and can reach (and in spots slightly exceed) the ``All demos'' reference, but remains below most stronger rollout/quality-informed baselines and sits closest to Demo-SCORE's band. We expect quality-based curation to outperform diversity-based curation when the dataset mixes high- and low-quality demonstrations (e.g., MH). The key trade is practical: \textsc{FAKTUAL} is model-free and lightweight, providing a fast ``good-enough'' curation pass when policy/rollout-dependent scoring or an auxiliary model is too expensive.}
    \label{fig:baseline_results}
    
\end{figure}

\subsection{Key observations and discussion}
As shown in Figures~\ref{fig:robomimic_results} and~\ref{fig:baseline_results}, in the low-demonstration regime diversity-maximizing curation can reduce success relative to random on harder tasks; this is not universal: on simpler tasks (e.g.\ Can), promoting diversity does not noticeably hurt performance. The demonstration budget at which diversity-aware curation overtakes random is task-dependent and typically higher for harder tasks, suggesting that discarding redundant/near-duplicate demonstrations becomes beneficial only once the budget is sufficiently large. This indicates that \emph{density} (Section~\ref{sec:intro}) matters more for harder tasks, while easier tasks tolerate lower density.

\textbf{Pick-and-place marker.} The marker position was sampled per episode from a fixed set of five locations, biased to the robot's left, and the pen-holder position independently from another set of five; our method mitigates this imbalance by approximately equalizing the distribution over marker positions. \textbf{Drawer open.} Five distinct lighting conditions were used during collection; our method did not distinguish between them, primarily because it operates on ViT-based image embeddings or low-dimensional observations (e.g.\ point clouds) of the objects of interest, which may not capture global illumination differences. \textbf{Organize objects.} Because this task is goal-conditioned, our method did not systematically favor trajectories ending in one particular goal configuration over others.

\section{Ablations} \label{sec:ablations}

\textbf{Number of signature levels vs.\ success rate.} The optimal number of signature levels depends on the nature of the dataset (and hence trajectory shapes). In Figure~\ref{fig:siglevels_results} (Appendix~\ref{app:add_results}) two patterns emerge as we increase levels: (1) success rate of the curated dataset improves with size \(m\) (right subplot); (2) success rate improves up to a point, then decreases (left subplot). Intuitively, the signature transform behaves like a discrete Fourier transform: Fourier captures low/high frequencies in increasing-order coefficients, while signature levels encode increasingly fine geometric information about the trajectory's shape. Thus, when the coarse shape carries most of the task-relevant information, a truncated signature is particularly useful; otherwise the full signature is important. When detailed shape carries little information (e.g.\ rough/noisy trajectories), the coarse shape is more informative. Throughout this ablation we use the random Fourier features method over dynamic programming methods for computing the truncated signature kernel.

\textbf{Value of \(p\) vs.\ success rate.} As with signature levels, the optimal \(p\) depends on the dataset; choosing \(p\in[0.5,0.9]\) yields improvement over the base policy, and even \(p=0\) performs well, because maximizing either entropy or determinant maximizes the intuitive notion of diversity.

\textbf{Signature kernel vs.\ global alignment kernel.} We benchmarked the signature kernel against the global alignment kernel~\citep{global_alignments}, another effective kernel for sequential data; curating with it yields the same performance as random selection.
\section{Conclusion and Limitations}\label{sec:limits}

We introduce signature-based entropy and diversity metrics over demonstrations and use them to develop \textsc{FAKTUAL}, a fast, model-free curation method for robot IL. Across simulated and real-world benchmarks, diversity-aware curation matches or outperforms the full dataset while producing more balanced datasets. Our work has several limitations.

(1) \textbf{Data quality:} Unlike quality-aware curation methods \citep{demo_score, DemInf, cupid}, our method does not guarantee high-quality trajectories; future work could combine both approaches. (2) \textbf{Datasets used:} Selecting the most diverse subset can be detrimental when a subset contains trajectories that are radically different but harmful for learning. None of our datasets contain such adversarial or pathological trajectories---only successful demonstrations. When such trajectories are present, quality-aware methods are likely more effective. (3) \textbf{Demo-to-performance relation:} Our approach does not establish strong causal relationships between individual training examples and resulting policy behavior, a connection \citet{cupid} aim to model directly. Hence our method remains susceptible to misattributing the true causes of policy success or failure. (4) \textbf{Entropy saturation and post-saturation performance:}
While our experiments suggest that the smallest subset achieving near-maximal entropy can
sometimes outperform both smaller subsets and the full dataset, we currently lack a strong
theoretical explanation for why this particular subset size is favorable, or why policy performance
can decrease when additional demonstrations are added after entropy has largely saturated in some cases.
Empirically, this behavior can make FAKTUAL appear to be filtering out ``bad'' demonstrations,
rather than simply selecting a small subset of ``good'' ones, despite the fact that the curation
criterion uses only diversity and contains no explicit notion of demonstration quality.
We hypothesize the following explanation:
Beyond dataset quality, balanced coverage (uniform representation across behavior modes, without unequal peaks) is also key \cite{parekh2026balancedbehaviorcloningimbalanced}, and entropy/diversity-based curation explicitly promotes this. Once a subset reaches near-uniform coverage, adding further points tends to reintroduce peaks from the raw dataset, which can hurt performance; even though curation is diversity-, not quality-, based, those demos skew the balance and are not necessarily ``bad''. Conversely, sufficient local density also matters: at low demo counts we are in a high-diversity, lowdensity regime, which is likewise harmful.
However, this interpretation remains a hypothesis supported by our empirical observations rather
than a consequence of the entropy objective, and developing a theory that jointly characterizes
diversity, balanced coverage, local density, and downstream policy performance remains an important
direction for future work.

Despite these limitations, our results highlight trajectory-kernel entropy
as a practical tool for measuring and promoting dataset diversity,
motivating future work (e.g., extensions combining diversity- and
quality-based curation at larger scales to handle datasets of variable
quality). A particularly promising direction is to design a \emph{mixed}
objective that jointly optimizes for quality and diversity, which would
help alleviate several drawbacks of a purely diversity-based criterion
(notably its susceptibility to adversarial or pathological trajectories
and its inability to filter out low-quality demonstrations on its own).
The principal difficulty in this direction is that defining a reliable
notion of demonstration quality is itself an open problem. The
mutual-information-based quality estimate of \cite{DemInf} is a natural
starting point, but it is a \emph{learned} metric that requires training
an auxiliary model, and therefore inherits precisely the kind of
computational overhead our method was designed to avoid. More broadly,
we conjecture that any candidate \emph{universal} metric for quantifying
the ``goodness'' of a robotics dataset should jointly account for all
three of: (1) demonstration quality, (2) local density around the
relevant behaviors and states, and (3) diversity (with balanced
coverage of trajectory modes). Our experiments provide concrete evidence
that the third component is indispensable. Any goodness metric that omits a diversity term would likely fail to
capture a property that empirically drives policy success, suggesting
that diversity is a necessary ingredient of any future unified curation criterion.



\clearpage
\bibliography{references}  


\clearpage
\appendices
\listofappendices
\appsection{Related Works} \label{app:related_work}

\paragraph{Dataset Diversity in Robotics}\citet{shi2025diversityneedscalablerobotic} consider a multi-task setting and show that:
(1) task diversity proves more critical than per-task demonstration quantity;
(2) multi-embodiment pre-training data is not strictly necessary for cross-embodiment transfer---models trained on high-quality single-embodiment data can transfer efficiently to different platforms and exhibit more favorable scaling behavior during fine-tuning than multi-embodiment pre-trained models; and
(3) expert diversity, stemming from individual operational preferences and stochastic variability in human demonstrations, can confound policy learning, with velocity multimodality (same action done slower/faster) being harmful to the performance and spatial multimodality being beneficial for performance. D3IL \cite{jia2024diversebehaviorsbenchmarkimitation} designs tasks with explicitly defined multiple successful behaviors and multiple human demonstrators, and evaluates IL algorithms using behavior entropy alongside success rate. They show that some advanced methods (especially diffusion-based transformer policies) can achieve both high diversity and high success. \citet{chi2024universalmanipulationinterfaceinthewild} collect demonstrations “in the wild” across diverse locations and object variants and shows that a single policy trained on this heterogeneous dataset generalizes to unseen environments and objects, including out-of-distribution settings (e.g., novel lighting, textures, surfaces).

\paragraph{Data Curation in Robotics}
Recent work has attempted to explicitly quantify the value of individual demonstrations for data curation~\cite{kuhar2023learning, DemInf, chen2025curating, demo_score}.
\citet{DemInf} estimates demonstration quality heuristic based on mutual information by training a VAE, and hence does not explicitly incorporate policy performance. However, this method can only filter existing data. \citet{chen2025curating} on the other hand train classifiers to distinguish successful from failed rollouts across policy checkpoints.
In contrast to these methods, \citet{cupid} directly quantify the causal influence of each demonstration on the policy’s expected return, yielding a signal that does not require access to both successes and failures.
DataMIL~\cite{dass2025datamil} applies datamodels to select from large multi-task datasets using an offline metric. SCIZOR~\cite{zhang2025scizorselfsupervisedapproachdata} is a self-supervised data curation approach to filter out suboptimal data, which hinder learning with undesirable actions, and redundant data, which dilute training with repetitive patterns. 

\paragraph{The signature (transform) and its corresponding kernel~\citep{lee2023signaturekernel, primerSK}} have proven to be highly effective in robotics and deep learning when data is naturally represented as a continuous paths or trajectories \citep{primerSK, deepsigtrans, Yang2022, signature_control}. Prior work has leveraged signature-based representations for diverse objectives, including motion planning in robotics where the signature kernel was used as similarity function between trajectories to produce diverse solutions for manipulation tasks \cite{PathSigtrajOpt}, human action recognition \cite{Yang2022}, and trajectory following or stabilization via signature methods \cite{signature_control}.
Signatures have also been applied to data-driven control, where trajectory information is encoded compactly to support learning and control design \cite{scampicchio2025rolesignaturetransformnonlinear}.

\citet{vendi_score} introduce the Vendi score which uses kernels to measure the diversity of a dataset. Subsequent works show that Vendi scores can be extended to measure information gain \cite{nguyen2025vendiinformationgainalternative}, and improve retrieval augmented generations \cite{rezaei2025vendi}. We extend the idea underlying the Vendi score to the robotic domain using Signature transforms and the Signature kernel. We also connect the Vendi score to Von Neumann and Shannon entropy under the Signature kernel.

\paragraph{Measuring Dataset Diversity in Machine Learning} Several commonly used diversity metrics depend on a reference distribution or dataset,
characterize diversity in terms of how well the reference is covered \cite{szegedy2016rethinking, heusel2017gans, cifka2018eval, preuer2018frechet, sajjadi2018assessing, naeem2020reliable}. Other existing metrics assess diversity using a pre-trained classifier, and thus require access to labeled datasets \cite{Salimans2016, Srivastava2017}. 
There are proposed metrics \cite{posada2020gait, shen2019mixture, benhenda2017chemgan} that use similarity scores to define diversity, such similarity scores can include average pairwise similarity score or the complement, the average dissimilarity. 
However, all these methods exhibit drawbacks, such as requiring a reference distribution, relying on manual inspection of samples, or failing to account for correlations between different features. \citet{vendi_score} addresses many of these challenges by introducing the Vendi Score, a principled metric for quantifying dataset diversity. The Vendi score has found several useful extensions throughout computer science for: measuring functional diversity \cite{dieng2025unified}, improving retrieval augmented generation \cite{rezaei2025vendi}, and measuring information gain \cite{nguyen2025vendiinformationgainalternative}. We extend the ideas underlying the Vendi Score to the robotics domain. 

\paragraph{Effect of Dataset Diversity on Model Performance in Machine Learning} \citet{yu-etal-2022-data} show that the diversification of training samples alleviates overfitting and improves model generalization and accuracy in natural language processing. In an imaging-through-scattering setting, \citet{zhang2024crossdatasetgeneralizationdeeplearning} show that increasing the diversity of the training dataset improves the deep neutral network’s ability to approximate the underlying linear physical mapping, thereby enabling robust generalization to unseen datasets. \citet{jung2025prismaticsynthesisgradientbaseddata} use a Vendi-style entropy, G-Vendi. They show that G-Vendi strongly correlates with how a (reasoning) model performs in unseen distributions when trained on that dataset.

\appsection{Demos to  Paths} \label{App:demo_to_paths}

For simplicity, we assume a single camera and a single robot; however, our method extends straightforwardly to settings with multiple cameras and robots.

\subsection{Method 1: Flatten and Stacked}
A demonstration \(d_i \in \mathcal{D}\) typically consists of various modalities such as states, actions, the robot’s perception data (e.g., RGB images from cameras, depth sensor measurements, etc.), rewards, and any additional auxiliary data recorded at each time step \(t \in T_i\), where \(T_i\) denotes the length of the demonstration. We denote the state at time \(t\) in \(d_i\) by \(s_i(t)\), the action by \(a_i(t)\), the RGB image by \(I^\text{RGB}_i(t)\), and the reward by \(R_i(t)\).

To convert a demonstration into a trajectory (or path), at each time step \(t \in T_i\) we first embed the RGB image into a fixed dimensional ViT~\citep{ViT} embedding \(e_i(t)\). Alternatively, one can also extract low-dimensional representation (such as point clouds, ground truth pose(s) etc.) of the objects of interest from the RGB images, we will denote this with \(l_i(t)\) (if \(l_i(t)\) is not vector, one can simply flatten it into a vector).
\footnote{All experiments in Section~\ref{sec:results} used the ViT embedding method except the results presented in Figure~\ref{fig:baseline_results}.} 
Let \(o_i(t) := e_i(t) \text{ or } l_i(t)\) denote the "observation" vector at time \(t \in T_i\). We then construct a feature vector by flattening and concatenating all components:
\[
    X_i(t) = 
    \begin{bmatrix}
        s_i(t) \\
        a_i(t) \\
        o_i(t) \\
        \vdots \\
        R_i(t)
    \end{bmatrix}.
\]
Now one can use the signature kernel over the trajectories \(X_i\).

\subsection{Method 2: Convex Combination}

The reader may observe that each data type (or modality) (e.g., states (\(s_i(t)\)), actions (\(a_i(t)\)), observations (\(o_i(t)\)), and rewards) forms a path or trajectory in its own right. Consequently, for each such path we can compute a signature kernel, and then take a convex combination of these kernels. Mathematically, let
\[
    k_{\text{state}}^\text{sig}, \; k_{\text{action}}^\text{sig}, \; k_{\text{observations}}^\text{sig}, \; \cdots, \; k_{\text{reward}}^\text{sig}
\]
denote the signature kernels corresponding to the state, action, observations, reward trajectories, etc.. We then define the combined kernel as
\begin{multline*}
    \hat k^\text{sig} = \alpha_{\text{state}} \, k_{\text{state}}^\text{sig} 
      + \alpha_{\text{action}} \, k_{\text{action}}^\text{sig} 
      + \alpha_{\text{observations}} \, k_{\text{observations}}^\text{sig}
      + \cdots
      + \alpha_{\text{reward}} \, k_{\text{reward}}^\text{sig},
\end{multline*}

where the coefficients satisfy
\begin{multline*}
    \alpha_{\text{state}}, \alpha_{\text{action}}, \alpha_{\text{observations}}, \cdots,  \alpha_{\text{reward}} \geq 0, 
    \quad \text{and} \quad 
    \alpha_{\text{state}} + \alpha_{\text{action}} + \alpha_{\text{observations}} + \cdots + \alpha_{\text{reward}} = 1.
\end{multline*}

This convex combination allows us to integrate information from multiple modalities in a principled way, while preserving the kernel’s validity.
\footnote{Although in our experiments this method did not yield better outcomes than the previous approach, it may scale more effectively with higher-dimensional settings (e.g., longer horizons) and additionally allows for individual kernel hyperparameter control for each modality.} Furthermore, \(\hat k^{sig}\) will inherit many useful properties of the signature kernel (e.g. universality).  

\subsection{Method 3: Low dim convex combination} \label{sec:method3}

Alternatively, one can use a simpler convex combination using the signature kernel and the RBF kernel defined as
\begin{equation}
k^{\mathrm{RBF}}(x, x') \;=\; \exp\!\biggl(-\frac{\lVert x - x'\rVert^2}{2h^2}\biggr).
\end{equation}
We use the first frame from each observation video, denoted \(I_i^{\mathrm{RGB}}(0)\), and compute either a ViT embedding or a low-dimensional observation from \(I_i^{\mathrm{RGB}}(0)\). Let \(p_i(t)\) denote the end-effector path of demonstration \(d_i\) at time \(t\). Using these quantities, we can construct a kernel matrix as follows:

\begin{equation*}
    \hat K^{mix} = \alpha_{1} K^{RBF} + \alpha_2 K^{sig}, \quad \sum \alpha_i = 1
\end{equation*}

where \(K_{i,j}^{RBF} = k^{rbf}(o_i(0), o_j(0))\) and \(K_{i,j}^{sig}= k^{sig}(p_t(t), p_j(t))\).

\subsection{Implementational practices}
\paragraph{Time coordinate.} It is important to note that, in practice, it can sometimes be useful to append a monotonic ``time'' coordinate to each path (e.g., \(s'_i(t) := (t, s_i(t))\)). This adjustment is necessary because the signature transform is reparameterization invariant.  

For example, consider the following constant paths:
\[
p_1(t) = \delta_a, \quad p_2(t) = \delta_b, 
\qquad 0 \leq t \leq 1, \; a,b \in [0,1], \; a \neq b,
\]
for which we obtain
\[
\varphi(p_1) = \varphi(p_2).
\]
However, after augmenting with time, we have
\begin{multline*}
p'_1(t) = (t, \delta_a), \quad p'_2(t) = (t, \delta_b), 
\qquad  0 \leq t \leq 1, \; a,b \in [0,1], \; a \neq b,
\end{multline*}
and in this case
\[
\varphi(p'_1) \neq \varphi(p'_2).
\]

This modification can sometimes be detrimental, and its impact is highly task- and dataset-dependent; in scenarios where reparameterization invariance is important, we do not recommend appending the time coordinate.

\paragraph{Sub-sampling.}
It is not always necessary to use the states, actions, and observations, etc. at all time steps \(t \in T\); for long-horizon tasks, it can be beneficial to sample them at regular intervals, especially to reduce the overall computational load.

\paragraph{Distractors.}
It is important to note that when using the method described in Subsection~\ref{sec:method3} or applying sub-sampling, distractors may not be captured adequately. This issue will particularly be pronounced when a distractor appears and disappears within a time interval shorter than the sampling period under sub-sampling, or when it only appears later in the trajectory while the method in Subsection~\ref{sec:method3} relies on earlier observations.

 \paragraph{Camera selection.}
If the dataset includes wrist-mounted cameras, we recommend not using them with the methods described above, as they can introduce substantial noise and make all demonstrations look diverse. In fact, we strongly advise using only a single camera that reliably captures the entire scene.

\appsection{Additional Background} \label{sec:background}

In this Appendix, we to provide some additional background on entropy and submodularity.

\subsection{Background on Entropy} \label{sec:entropy}
Entropy has long been used as a measure of randomness in a system across many fields, including physics, ecology, signal processing, information theory, and machine learning. It is often interpreted as the amount of information required to “describe’’ a system—the greater the randomness, the more information is needed to fully characterize it. 

In statistical physics, the (Gibbs) entropy~\cite{Gibbs_entropy} is defined as
\(
H_{\text{gibbs}} = -k_B \sum_{i} p_i \ln p_i,
\)
where \(p_i\) denotes the probability of each microscopic configuration \(x\) of a system that is consistent with the macroscopic constraints (such as total energy, volume, and number of particles), and \(k_B\) is Boltzmann’s constant. 
Shannon’s entropy~\citep{shannon_entropy} generalizes the Gibbs entropy, leading to the definition of information entropy as
\begin{equation} \label{eq:shannon_entropy}
    H_\text{shannon}(X) = -\sum_{x \in \mathcal{X}} p(x) \ln p(x),
\end{equation}
where \(X\) is a discrete random variable with support \(\mathcal{X}\) and \(p(x) = \mathbb{P}(X = x)\).
In general, the Shannon entropy measures the average uncertainty associated with the outcome of \(X\). In other words, the higher the Shannon entropy, the more uniform and spread out the distribution is.

\subsubsection{Von Neumann Entropy}
The \emph{von Neumann entropy}~\citep{Neumann_entropy, Neumann_entropy_translation} is the quantum analogue of Shannon entropy: it extends the notion of information-theoretic uncertainty from classical probability distributions to quantum states. Speaking very informally, the von Neumann entropy is defined as 
\begin{equation}
    H_\text{von neumann} = -\text{Tr}(\rho\ln\rho),
\end{equation}
where \(\rho\) is the density matrix, whose rows and columns correspond to the (quantum) states of the systems. The diagonal entries capture the probabilities of the system being in each state, while the off-diagonal entries reflect the “mixing” or coherence between different states. 

\subsection{Background on Submodularity} \label{App:submodular}

In this section, we aim to provide a background on submodular functions.

\begin{definition}[Subset Function] 
    Given a ground set \(\mathcal{U}\), a \emph{subset function} is a function 
    \[
        f : 2^{\mathcal{U}} \to \mathbb{R},
    \]
    that maps each subset of \(\mathcal{U}\) to a real value.
\end{definition}

\begin{definition} [Monotone Non-decreasing Function]
    A (subset) function is \emph{monotone non-decreasing} if 
    \[
        \forall A \subseteq B \subseteq \mathcal{U} \; : \; f(A) \leq f(B).
    \]
\end{definition}

\begin{definition} [Submodular Function] \label{def:submodular}
    A (subset) function \(f : 2^{\mathcal{U}} \to \mathbb{R}\) is \emph{submodular} if
    \[
        f(A+x)-f(A) \geq f(B+x)-f(B), \quad \forall A \subseteq B \subseteq \mathcal{U}, \; x \in \mathcal{U} \backslash B.
    \]
\end{definition}
\appsection{Signature Kernel} \label{app:sigkern}

\subsection{Signature Transform}

\begin{definition} \textbf{Signature transform of a path (trajectory)~\citep{sigPDE}.}
\label{sigT_formal}
Let \(V\) be a Banach space and consider the tensor algebra of formal power series
\[
T((V)) := \prod_{k\geq 0} V^{\otimes k},
\]
equipped with the natural operations induced by the tensor product and with unit element \((1,0,0,\ldots)\).

Let \(x:I\to V\) be a continuous path on a compact interval \(I\subset\mathbb{R}\) with finite \(p\)-variation for some \(p<2\).
For any subinterval \([s,t]\subset I\), the \emph{signature} of \(x\) over \([s,t]\) is the element of \(T((V))\) defined by the sequence of iterated integrals
\begin{multline}
\varphi(x)_{s,t}
\;=\;
\Bigl(
1,\;
\int_{s<u_1<t} \mathrm{d}x(u_1),\;
\ldots,\; 
\int_{s<u_1<\cdots<u_k<t} \mathrm{d}x(u_1)\otimes\cdots\otimes \mathrm{d}x(u_k),\;
\ldots
\Bigr).
\end{multline}
The signature provides a feature representation of the path and is invariant to reparameterization, i.e.,
\(\varphi(x)=\varphi(x\circ\theta)\) for any increasing reparameterization \(\theta\).
\end{definition}

\subsubsection{Properties of the Signature Transform}

\paragraph{Injectivity}
The path signature can be viewed as a canonical feature map that embeds a multivariate path into an (infinite) sequence of iterated integrals.
It is injective on the class of non--tree-like paths: two such paths cannot share the same full signature.
In practice, non--tree-likeness is typically enforced by augmenting the path with a monotone ``time'' coordinate.

\paragraph{Time-reversal duality}
The signature also satisfies a \emph{time-reversal duality} relation:
\[
\varphi(x)_{a,b}\;\otimes\;\varphi(\overleftarrow{x})_{a,b} \;=\; 1,
\]
where \(1\) denotes the unit element and \(\overleftarrow{x}(t) := x(a+b-t)\).

\paragraph{Chen's identity}
A central property of signatures is their multiplicative behavior under concatenation.
If a path is decomposed into successive segments, then the signature of the entire path can be obtained by tensor-multiplying the signatures of the segments.
This is formalized by Chen's identity.

\begin{theorem}[Chen's identity]
Consider the same setting as Definition~\ref{sigT_formal}, then for any \(a\le s\le u\le t\le b\),
\[
\varphi(x)_{s,t}
\;=\;
\varphi(x)_{s,u}
\;\otimes\;
\varphi(x)_{u,t}
\quad\in T((V)).
\]
Alternatively, for each level \(k\ge 0\),
\[
\varphi(x)_{s,t}^{k}
\;=\;
\sum_{i+j=k}
\varphi(x)_{s,u}^{i}
\;\otimes\;
\varphi(x)_{u,t}^{j},
\]
where \(\varphi(x)_{a,b}^{\ell}\in V^{\otimes \ell}\) is the \(\ell\)-th iterated integral.
\end{theorem}

In Hopf-algebraic terms, \(\varphi(x)\) is \emph{group-like} in the tensor algebra, i.e., its coproduct satisfies
\[
\Delta\,\varphi(x) \;=\; \varphi(x)\otimes\varphi(x).
\]

\paragraph{Factorial decay of signature terms}
For a path \(x\) of finite \(p\)-variation with \(p>1\), the norms of the level-\(k\) iterated integrals exhibit factorial decay~\cite{sigisDense2}:
\[
\left\|
\int_{s<u_1<\cdots<u_k<t}
\mathrm{d}x(u_1)\otimes\cdots\otimes \mathrm{d}x(u_k)
\right\|_{V^{\otimes k}}
\;\le\;
\frac{\|x\|_{p,[s,t]}^{k}}{k!},
\]
where \(\|x\|_{p,[s,t]}\) denotes the \(p\)-variation of \(x\) over \([s,t]\).
Consequently, the signature (and signature kernel) can often be well-approximated by truncation at a finite level \(L\in\mathbb{N}\).
The \emph{truncated signature} of level \(L\) is
\begin{multline}
\varphi^{L}(x)_{s,t}
\;=\;
\Biggl(
1,\,
\int_{s<u_1<t}\mathrm{d}x(u_1),\,
\ldots,\; 
\int_{s<u_1<\cdots<u_L<t}
\mathrm{d}x(u_1)\otimes\cdots\otimes \mathrm{d}x(u_L)
\Biggr)
\in
\bigoplus_{k=0}^{L} V^{\otimes k},
\end{multline}
where \(\oplus\) denotes the direct sum.

\subsection{Signature Kernel}
We now introduce the signature kernel, which uses the signature transform as a feature map to define a positive-definite kernel on paths.

\begin{definition}\textbf{Signature kernel.} \label{def:formal_sig_kern}
Consider the same setting as Definition~\ref{def:formal_sig_kern}, then the \emph{signature kernel} \(k^{\mathrm{sig}}(x,y)\colon I\times J\to\mathbb{R}\) is defined as
\begin{equation}
k^{\mathrm{sig}}_{s,t}(x,y)
\;=\;
\bigl\langle \varphi(x)_{u,s},\, \varphi(y)_{v,t} \bigr\rangle,
\end{equation}
where \(\varphi(x)_{u,s}\) and \(\varphi(y)_{v,t}\) denote the signatures of \(x\) and \(y\) over the subintervals \([u,s]\) and \([v,t]\), respectively, and \(\langle\cdot,\cdot\rangle\) denotes the canonical inner product on the (truncated or completed) tensor algebra.
\end{definition}

\subsubsection*{Computing the untruncated signature kernel}

Leveraging the \emph{kernel trick} of~\citep{sigPDE}, the (untruncated) signature kernel can be evaluated efficiently by solving a simple PDE.

\begin{theorem}\label{maintheorem_formal}
Let \(I=[u,u']\) and \(J=[v,v']\) be compact intervals, and let \(x\in\mathcal{C}^1(I,V)\) and \(y\in\mathcal{C}^1(J,V)\).
Then the signature kernel \(k^{\mathrm{sig}}(x,y)\) satisfies the linear second-order hyperbolic partial differential equation
\[
\frac{\partial^2}{\partial s\,\partial t}\,k^{\mathrm{sig}}_{s,t}(x,y)
\;=\;
\bigl\langle \dot{x}(s),\,\dot{y}(t)\bigr\rangle_{V}\,
k^{\mathrm{sig}}_{s,t}(x,y),
\]
with boundary conditions
\[
k^{\mathrm{sig}}_{u,t}(x,y)=1 \quad \forall\, t\in J,
\qquad
k^{\mathrm{sig}}_{s,v}(x,y)=1 \quad \forall\, s\in I.
\]
\end{theorem}

This PDE characterization of the signature kernel enables efficient evaluation using standard numerical schemes, including finite-difference and finite-element discretizations, as well as other appropriate schemes.

Computing the full \((n\times n)\) Gram matrix for \(n\) trajectories/paths of length \(T\) in \(\mathbb{R}^d\) via dynamic programming has time complexity \(O\bigl(n^2\,T^2\,d^c\bigr)\), where \(c\) depends on the truncation level of the signature transform.
Using the kernel trick in Theorem~\ref{maintheorem_formal}, the kernel between two paths can be evaluated in as little as \(O\bigl(T^2\,d\bigr)\).
Forming the entire Gram matrix still requires \(n^2\) pairwise evaluations, and thus the overall cost remains \(O\bigl(n^2\,T^2\,d\bigr)\).
However, the underlying PDE formulation is amenable to efficient parallelization.
In particular, an explicit finite--difference implementation on a sufficiently provisioned GPU can reduce the per-pair computation to effectively linear scaling in the trajectory length in practice, yielding an effective wall-clock complexity of
\(
  O\bigl(T\,d\bigr)
\)
assuming the GPU can support the required level of parallelism~\citep{sigPDE}.

\appsection{Approximating the Signature Kernel with Random Fourier Features} \label{app:rfsf}

This appendix gives a self-contained exposition of the Random Fourier Signature
Feature (RFSF) construction of \citet{rffsig} (which we highly recommend reading before
consulting this Appendix), which is the
underlying machinery used in FAKTUAL whenever signature kernels are evaluated
via random features instead of dynamic programming (e.g. Section~\ref{sec:ablations},
``Number of signature levels vs.\ success rate'').

\paragraph{Notation.} Throughout this appendix we write
$\mathcal{X} \subseteq \mathbb{R}^{d}$ for a (compact, convex) state space, and
$\mathcal{X}_{\mathrm{seq}} := \{x = (x_{1}, \dots, x_{\ell_x}) : x_{i} \in
\mathcal{X},\ \ell_x \in \mathbb{Z}_{+}\}$ for the set of finite-length
sequences in $\mathcal{X}$. We let $\delta x_{i} := x_{i+1} - x_{i}$ denote the
first-order forward difference, $\|x\|_{1\text{-var}} := \sum_{i=1}^{\ell_x-1}
\|\delta x_{i}\|_{2}$ the $1$-variation, and
\[
  \Delta_{m}(n) := \{(i_{1}, \dots, i_{m}) \in [n]^{m} : 1 \le i_{1} < \cdots 
  i_{m} \le n\}, \qquad [n] = \{1, \dots, n\},
\]
the set of strictly increasing $m$-indices. Given a static kernel
$k : \mathcal{X} \times \mathcal{X} \to \mathbb{R}$, the second-order
cross-differencing operator is defined by
\[
  \delta^{2}_{i,j}\, k(x_{i}, y_{j})
  := k(x_{i+1}, y_{j+1}) - k(x_{i+1}, y_{j}) - k(x_{i}, y_{j+1}) + k(x_{i},
  y_{j}).
\]
Recalling Appendix~\ref{app:sigkern}, the
$L$-truncated signature kernel
$k^{sig,\le L} : \mathcal{X}_{\mathrm{seq}} \times \mathcal{X}_{\mathrm{seq}}
\to \mathbb{R}$ admits the discrete expansion
\begin{equation}\label{eq:sig-kernel-discrete}
  k^{sig,\le L}(x, y)
  \;=\; \sum_{m=0}^{L}\,
  \sum_{\substack{i \in \Delta_{m}(\ell_x - 1) \\ j \in \Delta_{m}(\ell_y - 1)}}
  \prod_{p=1}^{m} \delta^{2}_{i_{p}, j_{p}}\, k(x_{i_{p}}, y_{j_{p}}).
\end{equation}
Although $k^{sig,\le L}$ admits a kernel trick of cost $O(\ell_x \ell_y(L + d))$
per pair, evaluating the full Gram matrix $K^{sig}$ over $N$ sequences of
length~$\ell$ costs $O(N^{2} \ell^{2}(L + d))$, which is the bottleneck the RFSF
construction below alleviates.

\subsection{Background on Random Fourier Features}\label{app:rff-background}

Random Fourier Features (RFF) random approximate any
continuous, bounded, translation-invariant kernel
$k : \mathbb{R}^{d} \times \mathbb{R}^{d} \to \mathbb{R}$ (i.e.\
$k(x, y) \equiv k(x - y)$). By Bochner's theorem, there exists a finite
non-negative measure $\Lambda$ on $\mathbb{R}^{d}$---the
\emph{spectral measure} of $k$---such that
\[
  k(x - y) \;=\; \int_{\mathbb{R}^{d}} e^{i w^{\top}(x - y)}\, d\Lambda(w).
\]
Without loss of generality (after rescaling by $k(0)$ if needed) we treat
$\Lambda$ as a probability measure. Drawing $w_{1}, \dots, w_{\tilde{d}}
\overset{iid}{\sim} \Lambda$, the RFF map is
\[
  \tilde{\varphi}(x) \;=\; \frac{1}{\sqrt{\tilde{d}}}\,
  \bigl(\cos(w_{1}^{\top} x), \sin(w_{1}^{\top} x), \dots,
        \cos(w_{\tilde{d}}^{\top} x), \sin(w_{\tilde{d}}^{\top} x)\bigr)^{\top}
  \in \mathbb{R}^{2\tilde{d}},
\]
and the random kernel $\tilde{k}(x, y) := \langle \tilde{\varphi}(x),
\tilde{\varphi}(y)\rangle$ satisfies
$\mathbb{E}[\tilde{k}(x, y)] = k(x, y)$ with exponentially fast uniform
concentration over compact sets. For
the Gaussian kernel $k(x, y) = \exp(-\|x - y\|_{2}^{2} / 2\sigma^{2})$, we have
$\Lambda = \mathcal{N}(0, \sigma^{-2} I_{d})$.

Throughout this appendix we assume that $\Lambda$ satisfies the standard
Bernstein moment condition: there exist $S, R > 0$ such that
\begin{equation}\label{eq:bernstein}
  \mathbb{E}_{w \sim \Lambda}\bigl[|w_{i}|^{2k}\bigr]
  \;\le\; \tfrac{k!\, S^{2}\, R^{k-2}}{2}, \qquad \forall\, i \in [d],\ k \ge 2.
\end{equation}
This is equivalent to $\Lambda$ being sub-Gaussian and is satisfied by the
Gaussian spectral measure above.

\subsection{Random Fourier Signature Features}\label{app:rfsf-features}

A naive plug-in of an RFF map into the signature construction of
Appendix~\ref{app:sigkern} only yields a \emph{biased} estimator of the
truncated signature kernel~\eqref{eq:sig-kernel-discrete}. To restore
unbiasedness, T\'oth et al.~\cite{rffsig} use $L$ \emph{independent}
RFF realizations---one per tensor level.

\begin{definition}[Random Fourier Signature Features]\label{def:rfsf}
  Let $W^{(1)}, \dots, W^{(L)} \overset{iid}{\sim} \Lambda^{\tilde{d}}$ be
  independent matrices of RFF weights, and define for each $m \in [L]$ the
  independent RFF map
  \[
    \tilde{\varphi}_{m}(x)
    \;=\; \frac{1}{\sqrt{\tilde{d}}}
    \bigl(\cos(W^{(m)\top} x), \sin(W^{(m)\top} x)\bigr) \in \mathbb{R}^{2\tilde{d}}.
  \]
  The $L$-truncated \emph{RFSF map} of a sequence $x \in \mathcal{X}_{\mathrm{seq}}$
  is $\tilde{\varphi}^{\le L}_{sig}(x) := (\tilde{\varphi}^{m}_{sig}(x))_{m=0}^{L}$
  with $\tilde{\varphi}^{0}_{sig}(x) := 1$ and, for $m \ge 1$,
  \[
    \tilde{\varphi}^{m}_{sig}(x)
    \;=\; \sum_{i \in \Delta_{m}(\ell_x - 1)}
    \bigl(\tilde{\varphi}_{1}(x_{i_{1}+1}) - \tilde{\varphi}_{1}(x_{i_{1}})\bigr)
    \otimes \cdots \otimes
    \bigl(\tilde{\varphi}_{m}(x_{i_{m}+1}) - \tilde{\varphi}_{m}(x_{i_{m}})\bigr).
  \]
  The associated \emph{RFSF kernel} is
  \[
    \tilde{k}^{sig,\le L}(x, y)
    \;=\; \bigl\langle \tilde{\varphi}^{\le L}_{sig}(x),
                      \tilde{\varphi}^{\le L}_{sig}(y) \bigr\rangle
    \;=\; \sum_{m=0}^{L}
    \sum_{\substack{i \in \Delta_{m}(\ell_x - 1) \\ j \in \Delta_{m}(\ell_y - 1)}}
    \prod_{p=1}^{m} \delta^{2}_{i_{p}, j_{p}}\, \tilde{k}_{p}(x_{i_{p}}, y_{j_{p}}),
  \]
  where $\tilde{k}_{p}(u, v) := \langle \tilde{\varphi}_{p}(u),
  \tilde{\varphi}_{p}(v)\rangle$ is the level-$p$ RFF kernel.
\end{definition}

By construction $\mathbb{E}\bigl[\tilde{k}^{sig,\le L}(x, y)\bigr] =
k^{sig,\le L}(x, y)$. \citet{rffsig} additionally give
uniform concentration of the random kernel to the true one.

\begin{theorem}[Uniform concentration of the RFSF kernel,
                {\cite[Thm.~3.2]{rffsig}}]\label{thm:rfsf-uniform}
  Let $k$ be a continuous, bounded, translation-invariant kernel on
  $\mathbb{R}^{d}$ whose spectral measure $\Lambda$ satisfies~\eqref{eq:bernstein},
  let $\mathcal{X} \subset \mathbb{R}^{d}$ be compact and convex, and fix a
  $1$-variation bound $V > 0$. Then there exist constants $C_{d,\mathcal{X}},
  \beta_{d,m,V} > 0$ (independent of $\tilde{d}$) such that for every level
  $m \in [L]$ and any $\epsilon > 0$,
  \[
    \mathbb{P}\!\left[\,
      \sup_{\substack{x, y \in \mathcal{X}_{\mathrm{seq}} \\
                      \|x\|_{1\text{-}\mathrm{var}},
                      \|y\|_{1\text{-}\mathrm{var}} \le V}}
      \bigl|k^{sig,m}(x, y) - \tilde{k}^{sig,m}(x, y)\bigr| \;\ge\; \epsilon
    \,\right]
    \;\le\; m \cdot \mathcal{B}_{d,m,V}(\epsilon, \tilde{d}),
  \]
  where
  \[
    \begin{gathered}
    \mathcal{B}_{d,m,V}(\epsilon, \tilde{d}) =\\[3pt]
    \begin{cases}
      \Bigl(C_{d,\mathcal{X}}\bigl(\tfrac{\beta_{d,m,V}}{\epsilon}\bigr)^{d/(d+1)} + d\Bigr)
      \exp\!\Bigl(-\tfrac{\tilde{d}}{2(d+1)(S^{2}+R)}
                  \bigl(\tfrac{\epsilon}{\beta_{d,m,V}}\bigr)^{2}\Bigr),
        & \epsilon < \beta_{d,m,V}, \\[1.5ex]
      \Bigl(C_{d,\mathcal{X}}\bigl(\tfrac{\beta_{d,m,V}}{\epsilon}\bigr)^{d/((d+1)m)} + d\Bigr)
      \exp\!\Bigl(-\tfrac{\tilde{d}}{2(d+1)(S^{2}+R)}
                  \bigl(\tfrac{\epsilon}{\beta_{d,m,V}}\bigr)^{1/m}\Bigr),
        & \epsilon \ge \beta_{d,m,V}.
    \end{cases}
    \end{gathered}
  \]
\end{theorem}

\noindent The tail is sub-Gaussian for small $\epsilon$ and
$(1/m)$-subexponential for large $\epsilon$; the regime change at
$\beta_{d,m,V}$ reflects the $m$-fold tensor product of RFF kernels in
Definition~\ref{def:rfsf}.

\paragraph{Computational cost.} The RFSF map can be evaluated for $N$ sequences
of length $\ell$ in $O\bigl(N \ell (Ld\tilde{d} + \tilde{d} + \tilde{d}^{2} +
\cdots + \tilde{d}^{L})\bigr)$ time. The polynomial blow-up in $\tilde{d}^{L}$
comes from storing the level-$m$ tensor $\tilde{\varphi}^{m}_{sig}(x) \in
\mathbb{R}^{(2\tilde{d})^{m}}$ and is prohibitive for moderate $L$ or
$\tilde{d}$; the next two subsections address this.

\subsection{Dimensionality reduction I: Diagonal projection (RFSF-DP)}
\label{app:rfsf-dp}

Expanding the inner RFF kernels in Definition~\ref{def:rfsf} yields
\begin{equation}\label{eq:rfsf-cos}
  \tilde{k}^{sig,m}(x, y)
  = \frac{1}{\tilde{d}^{m}}
  \sum_{q_{1}, \dots, q_{m}=1}^{\tilde{d}}
  \sum_{\substack{i \in \Delta_{m}(\ell_x - 1) \\ j \in \Delta_{m}(\ell_y - 1)}}
  \prod_{p=1}^{m} \delta^{2}_{i_{p}, j_{p}}
  \cos\!\bigl(w^{(p)\top}_{q_{p}}(x_{i_{p}} - y_{j_{p}})\bigr),
\end{equation}
where $w^{(p)}_{q}$ is the $q$-th column of $W^{(p)}$. The outer sum runs over
the $\tilde{d}^{m}$ multi-indices in $[\tilde{d}]^{m}$. Restricting it to the
diagonal $\{(q, q, \dots, q) : q \in [\tilde{d}]\}$ yields an unbiased estimator
with only $\tilde{d}$ outer terms~\cite{rffsig}.

\begin{definition}[Diagonally-projected RFSF]\label{def:rfsf-dp}
  With the weights of Definition~\ref{def:rfsf}, define for $q \in [\tilde{d}]$
  and $m \in [L]$ the per-column RFF map
  \[
    \hat{\varphi}_{m,q}(x) := \bigl(\cos(w^{(m)\top}_{q} x),
                                    \sin(w^{(m)\top}_{q} x)\bigr) \in \mathbb{R}^{2}.
  \]
  The $L$-truncated \emph{RFSF-DP map}
  $\hat{\varphi}^{\le L}_{sig}(x) = (\hat{\varphi}^{m}_{sig}(x))_{m=0}^{L}$ is
  given by $\hat{\varphi}^{0}_{sig}(x) := 1$ and, for $m \ge 1$,
  \[
    \hat{\varphi}^{m}_{sig}(x)
    = \frac{1}{\sqrt{\tilde{d}}}
    \Bigl(\sum_{i \in \Delta_{m}(\ell_x - 1)}
          \delta \hat{\varphi}_{1,q}(x_{i_{1}})
          \otimes \cdots \otimes
          \delta \hat{\varphi}_{m,q}(x_{i_{m}})\Bigr)_{q=1}^{\tilde{d}},
  \]
  with $\delta \hat{\varphi}_{p,q}(x_{t}) := \hat{\varphi}_{p,q}(x_{t+1}) -
  \hat{\varphi}_{p,q}(x_{t})$. The corresponding kernel is
  \[
    \tilde{k}^{sig,\le L,\mathrm{DP}}(x, y)
    = \sum_{m=0}^{L} \frac{1}{\tilde{d}} \sum_{q=1}^{\tilde{d}}
      \sum_{\substack{i \in \Delta_{m}(\ell_x - 1) \\
                      j \in \Delta_{m}(\ell_y - 1)}}
      \prod_{p=1}^{m} \delta^{2}_{i_{p}, j_{p}}\,
      \hat{k}_{p,q}(x_{i_{p}}, y_{j_{p}}),
  \]
  where $\hat{k}_{p,q}(u, v) := \langle \hat{\varphi}_{p,q}(u),
  \hat{\varphi}_{p,q}(v)\rangle$.
\end{definition}

Since each per-column tensor lives in $(\mathbb{R}^{2})^{\otimes m}$, the
overall feature dimension is $\tilde{d}\,(2^{L+1} - 1)$---linear in $\tilde{d}$,
exponential in $L$ but with a small base of $2$.

\begin{theorem}[Concentration of the RFSF-DP kernel,
                {\cite[Thm.~3.5]{rffsig}}]\label{thm:rfsf-dp}
  Under~\eqref{eq:bernstein}, for every $m \in [L]$, $x, y \in
  \mathcal{X}_{\mathrm{seq}}$, and $\epsilon > 0$,
  \[
    \mathbb{P}\!\bigl[\,
      |\tilde{k}^{sig,m,\mathrm{DP}}(x, y) - k^{sig,m}(x, y)| \ge \epsilon
    \,\bigr]
    \;\le\;
    2 \exp\!\biggl(-\tfrac{1}{4}\min\!\Bigl\{
       \Bigl(\tfrac{\sqrt{\tilde{d}}\,\epsilon}{2 C_{d,m,x,y}}\Bigr)^{2},
       \Bigl(\tfrac{\tilde{d}\,\epsilon}{\sqrt{8}\,C_{d,m,x,y}}\Bigr)^{1/m}
    \Bigr\}\biggr),
  \]
  where
  \[
    C_{d,m,x,y}
    \le \sqrt{8e^{4}(2\pi)^{1/4} e^{1/24}}\,
    \bigl(4 e^{3}\|x\|_{1\text{-}\mathrm{var}}\|y\|_{1\text{-}\mathrm{var}}/m\bigr)^{m/2}
    \sqrt{(2d \max(S, R))^{m} + (L^{2}/\ln 2)^{m}}.
  \]
\end{theorem}

\paragraph{Computational cost.} The RFSF-DP map costs
\[
  O\bigl(N \ell\, \tilde{d}\,(Ld + 2L)\bigr),
\]
i.e.\ \emph{linear} in both sequence length $\ell$ and RFF dimension
$\tilde{d}$.

\subsection{Dimensionality reduction II: Tensor random projection (RFSF-TRP)}
\label{app:rfsf-trp}

An alternative is to apply a CP-rank-1 tensor random
projection to each level-$m$ tensor
of the RFSF map. By the Johnson--Lindenstrauss
lemma, inner products are approximately preserved
while the parameter count drops from $O((2\tilde{d})^{m})$ to
$O(m\tilde{d} d)$.

\begin{definition}[Tensor-random-projected RFSF]\label{def:rfsf-trp}
  With $W^{(1)}, \dots, W^{(L)} \overset{iid}{\sim} \Lambda^{\tilde{d}}$ as in
  Definition~\ref{def:rfsf}, let
  \[
    P^{(m)} = \bigl(p^{(m)}_{1}, \dots, p^{(m)}_{\tilde{d}}\bigr)
    \in \mathbb{R}^{2\tilde{d} \times \tilde{d}}, \qquad
    p^{(m)}_{q} \overset{iid}{\sim} \mathcal{N}(0, I_{2\tilde{d}}),
    \quad m \in [L].
  \]
  The $L$-truncated \emph{RFSF-TRP map} $\check{\varphi}^{\le L}_{sig}(x) =
  (\check{\varphi}^{m}_{sig}(x))_{m=0}^{L}$ is given by
  $\check{\varphi}^{0}_{sig}(x) := 1$ and, for $m \ge 1$,
  \[
    \check{\varphi}^{m}_{sig}(x)
    = \frac{1}{\sqrt{\tilde{d}}}
    \sum_{i \in \Delta_{m}(\ell_x - 1)}
    \bigl(P^{(1)\top} \delta\tilde{\varphi}_{1}(x_{i_{1}})\bigr)
    \odot \cdots \odot
    \bigl(P^{(m)\top} \delta\tilde{\varphi}_{m}(x_{i_{m}})\bigr)
    \in \mathbb{R}^{\tilde{d}},
  \]
  where $\delta\tilde{\varphi}_{p}(x_{t}) := \tilde{\varphi}_{p}(x_{t+1}) -
  \tilde{\varphi}_{p}(x_{t})$ and $\odot$ denotes the Hadamard product. The
  corresponding kernel is
  \[
    \tilde{k}^{sig,\le L,\mathrm{TRP}}(x, y)
    = \sum_{m=0}^{L} \frac{1}{\tilde{d}} \sum_{q=1}^{\tilde{d}}
      \sum_{\substack{i \in \Delta_{m}(\ell_x - 1) \\
                      j \in \Delta_{m}(\ell_y - 1)}}
      \prod_{p=1}^{m}
      \bigl\langle p^{(p)}_{q}, \delta\tilde{\varphi}_{p}(x_{i_{p}})\bigr\rangle
      \bigl\langle p^{(p)}_{q}, \delta\tilde{\varphi}_{p}(y_{j_{p}})\bigr\rangle.
  \]
\end{definition}

The Hadamard form in Definition~\ref{def:rfsf-trp} lets the RFSF and TRP
projections be \emph{fused} during computation: one never instantiates the full
RFSF tensor.

\begin{theorem}[Concentration of the RFSF-TRP kernel,
                {\cite[Thm.~3.8]{rffsig}}]\label{thm:rfsf-trp}
  Under~\eqref{eq:bernstein}, for every $m \in [L]$, $x, y \in
  \mathcal{X}_{\mathrm{seq}}$, and $\epsilon > 0$,
  \[
    \mathbb{P}\!\bigl[\,
      |\tilde{k}^{sig,m,\mathrm{TRP}}(x, y) - \tilde{k}^{sig,m}(x, y)|
      \ge \epsilon
    \,\bigr]
    \;\le\;
    C_{d,\Lambda} \exp\!\biggl(
      -\Bigl(
        \tfrac{m^{2}\,\tilde{d}^{1/(2m)}\,\epsilon^{1/m}}
              {2 \sqrt{2 e^{3} R}\,
               \|x\|_{1\text{-}\mathrm{var}}\|y\|_{1\text{-}\mathrm{var}}}
      \Bigr)^{1/2}
    \biggr),
  \]
  where $C_{d,\Lambda} := 2\bigl(1 + \tfrac{S}{2R} + \tfrac{S^{2}}{4R^{2}}
  \bigr)^{d}$.
\end{theorem}

Combining Theorem~\ref{thm:rfsf-trp} with Theorem~\ref{thm:rfsf-uniform} via the
triangle inequality and a union bound yields concentration of
$\tilde{k}^{sig,m,\mathrm{TRP}}$ around the \emph{true} signature kernel
$k^{sig,m}$.

\paragraph{Computational cost.} The RFSF-TRP map costs
\[
  O\bigl(N \ell\, \tilde{d}\,(Ld + \tilde{d})\bigr),
\]
linear in $\ell$ and almost linear in $\tilde{d}$, with a mild quadratic term
from the projection step.

\subsection{Summary and use in FAKTUAL}\label{app:rfsf-summary}

Table~\ref{tab:rfsf-summary} summarizes the trade-offs between feature
dimension, time per sequence, and concentration rate for the three variants.

\begin{table}[h]
\centering
\small
\begin{tabular}{lccc}
\toprule
Method   & Feature dim.                       & Time / sequence
         & Concentration rate in $\tilde{d}$ \\
\midrule
RFSF     & $\sum_{m=0}^{L}(2\tilde{d})^{m}$   & $O\bigl(\ell\, \tilde{d}^{L}\bigr)$
         & sub-Gaussian / $(1/m)$-subexp. \\
RFSF-DP  & $\tilde{d}\,(2^{L+1}-1)$           & $O\bigl(\ell\, \tilde{d}\,(Ld + 2L)\bigr)$
         & $(1/m)$-subexp. \\
RFSF-TRP & $L\,\tilde{d}$                     & $O\bigl(\ell\, \tilde{d}\,(Ld + \tilde{d})\bigr)$
         & $(1/4m)$-subexp. \\
\bottomrule
\end{tabular}
\caption{Computational/statistical trade-offs for the three RFSF variants of
\cite{rffsig}.}
\label{tab:rfsf-summary}
\end{table}

In FAKTUAL we use these random feature kernels whenever the dataset is large
enough that the $O(N^{2} \ell^{2})$ cost of dynamic-programming--based
signature kernels becomes a bottleneck for assembling the Gram matrix $K^{sig}$
that defines the signature Shannon/von Neumann entropies of Definitions~3
and~4. The (approximate) entropy- and determinant-maximizing subsets
$\hat{\mathcal{D}}^{\mathrm{entropy}}_{m}$ and $\hat{\mathcal{D}}^{\mathrm{det}}_{m}$
(Definition~5) are then computed from the empirical Gram matrix
$\tilde{K}^{sig}_{i,j} := \tilde{k}^{sig,\le L,\mathrm{DP}}(d_{i}, d_{j})$, and
Theorems~\ref{thm:rfsf-uniform} and~\ref{thm:rfsf-dp} guarantee that this
approximation deviates entrywise from the true Gram matrix by at most
$\epsilon$ with probability decaying (sub-)exponentially in $\tilde{d}$. The
ablation in Appendix~\ref{app:add_results} uses this RFSF-TRP
implementation throughout.
\appsection{Algorithms for computing the maximum \texorpdfstring{\(m\)}{m}-subset} \label{app:algos}

Using the same notation as in Sections~\ref{sec:dataset_diversity} and~\ref{sec:dataset_curation}, 
the maximum coverage problem in our setting can be formulated as the following integer program:
\begin{equation*}
\begin{aligned}
\text{maximize} \quad   & f(D) \\
\text{subject to} \quad & D \subseteq \mathcal{D}, \\
                        & |D| = m \text{ and } m \leq |\mathcal D|
\end{aligned}
\end{equation*}
where consider \(f(\cdot) = S(\cdot)\) or \(f(\cdot) = \Delta(\cdot)\).

In practice, it is often convenient to use the regularized objective
\begin{multline*}
    \Delta'(D)
    := \log\det\!\bigl(K^{\text{sig}} + \mu I_{m\times m}\bigr),
    K^{\text{sig}}_{i,j} = k^{\text{sig}}(d_i,d_j), \; d_i,d_j \in D, \quad \mu \in \mathbb{R}_{>0},
\end{multline*}
instead of \(\Delta\). This modification does not change the solution of the optimization problem. 
However, it is advantageous in two respects: \(\Delta'\) is more numerically stable, and \(\Delta'\) is submodular (see Definition~\ref{def:submodular}).

In general, integer programming problems are NP-hard. Consequently, computing an exact solution is feasible only for very small datasets. In such cases, one may employ branch-and-bound methods to solve the problem exactly~\citep{integer_combinatorial_optimization, integer_programing}. For larger instances, integer programs can often be approximated in polynomial time using various algorithmic techniques (see, e.g.,~\citep{approximation_algorithms}). In this work, we focus on a set of simple and easy-to-implement algorithms.

\subsection{Greedy Local Search}
Given the ground set \(\mathcal D\) and a objective \(f:2^V \to \mathbb{R}\) and cardinality bound \(m\), we first greedily build a size-\(m\) set \(X\) by repeatedly adding the element with maximum marginal gain. Then we optionally apply local search: while there exists a single-element swap that improves \(f\), we perform it.

\begin{algorithm}[ht]
  \caption{Greedy Local Search (cardinality \(m\))}
  \PaperBegin{algorithm1}{algorithm}{PaperTeal}{0pt}{0pt}
  \begin{algorithmic}[1]
    \State \textbf{Input:} ground set \(\mathcal{D}\), objective \(f(\cdot)\), integer \(m\)
    \State \(X \gets \emptyset\)
    \Comment{greedy initialization}
    \For{\(i = 1\) to \(m\)}
      \State choose \(e \in \mathcal{D} \setminus X\) maximizing 
             \(G(e \mid X) = f(X \cup \{e\}) - f(X)\)
      \If{\(G(e \mid X) \le 0\)}
        \State \textbf{break}
      \EndIf
      \State \(X \gets X \cup \{e\}\)
    \EndFor
    \Comment{1-swap local search}
    \While{there exist \(e \in X\), \(e' \in \mathcal{D} \setminus X\) s.t.
           \(f(X \setminus \{e\} \cup \{e'\}) > f(X)\)}
      \State \(X \gets X \setminus \{e\} \cup \{e'\}\)
    \EndWhile
    \State \textbf{return} \(X\)
  \end{algorithmic}
  \PaperEnd{algorithm1}
\end{algorithm}

\paragraph{Approximation guarantee}
For a nonnegative, monotone, submodular function \(f\) under a simple cardinality constraint \(|X| \le m\), the classical greedy algorithm \citep{local_greedy_sreach, local_gready_sreach2} achieves a \((1 - 1/e)\)-approximation, and any subsequent local search that only performs improving moves preserves this guarantee. If \(f\) is non-submodular this algorithm provides no guarantee, however in practice this will preform fairly well.

\subsection{Stochastic Greedy}
Given a ground set \(\mathcal{D}\) of size \(n\), a objective function \(f\) and
a cardinality constraint \(m\), the stochastic greedy algorithm maintains a
solution set \(X\). At each of \(m\) iterations it samples a random subset of
the remaining elements of size
\(\lceil (n/m)\log(1/\varepsilon) \rceil\), evaluates marginal gains only on this
sample, and adds the sampled element with the largest marginal gain.

\begin{algorithm}[ht]
  \caption{Stochastic Greedy (cardinality \(m\))}
  \PaperBegin{algorithm2}{algorithm}{PaperTeal}{0pt}{0pt}
  \begin{algorithmic}[1]
    \State \textbf{Input:} ground set \(\mathcal{D}\), objective \(f(\cdot)\), integers \(n = |\mathcal{D}|\), \(m\), parameter \(\varepsilon \in (0,1)\)
    \State \(X \gets \emptyset\)
    \State \(X_{\text{val}} \gets f(X)\)
    \State \(r \gets \lceil (n/m)\log(1/\varepsilon) \rceil\)
    \For{\(i = 1\) to \(m\)}
      \State \(A \gets \mathcal{D} \setminus X\) \Comment{remaining elements}
      \If{\(A = \emptyset\)} \State \textbf{break} \EndIf
      \State sample \(R \subseteq A\) uniformly without replacement with
             \(|R| = \min\{|A|, r\}\)
      \State \(e^\star \gets \arg\max_{e \in R} \bigl(f(X \cup \{e\}) - X_{\text{val}}\bigr)\)
      \State \(X \gets X \cup \{e^\star\}\)
      \State \(X_{\text{val}} \gets f(X)\)
    \EndFor
    \State \textbf{return} \(X\)
  \end{algorithmic}
  \PaperEnd{algorithm2}
\end{algorithm}

\paragraph{Approximation guarantee.}
Let \(f\) be nonnegative, monotone, and submodular. For any \(\varepsilon \in (0,1)\), the stochastic greedy algorithm above, with \(r = \lceil (n/m)\log(1/\varepsilon) \rceil\), returns a random set \(X\) such that \(\mathbb{E}[f(X)] \ge (1 - 1/e - \varepsilon)\,f(X^\star)\), where \(X^\star\) is an optimal solution of size at most \(m\) \citep{stochastic_greedy}. The expected number of function evaluations is \(O\bigl(n \log(1/\varepsilon)\bigr)\), independent of \(m\).

\subsection{Sampling an \texorpdfstring{\(m\)}{m}-DPP on \texorpdfstring{\(\mathcal{D}\)}{D}.}

Another principled way to select a diverse subset of elements from the ground set
\(\mathcal{D}\) is via determinantal point processes (DPPs) \citep{dpp}.\footnote{We choose only provide a very brief introduction, DPPs in general require a lot of mathematical machinery and background that is beyond the scope of this article, hence we strongly recommend the interested reader to consult \citep{dpp}.} In an L-ensemble DPP
with kernel \(L \propto K^{\text{sig}} + \mu I\), the probability of a subset
\(X \subseteq \mathcal{D}\) is proportional to \(\det(L_{X,X})\), so that
maximizing \(\Delta(X)\) (or \(\Delta'(X)\)) is exactly the maximum a posteriori
(MAP) inference problem for this DPP. We next describe the corresponding DPP-based
algorithm for selecting a diverse subset of size \(m\).

Let \(\mathcal{D} = \{d_1,\dots,d_N\}\) be the ground set and
\(\widetilde{K}^{\text{sig}} \in \mathbb{R}^{N\times N}\) the full similarity matrix,
\[
  \widetilde{K}^{\text{sig}}_{i,j} = k^{\text{sig}}(d_i,d_j),
  \qquad d_i,d_j \in \mathcal{D}.
\]
We define the L-ensemble kernel
\[
  L := \widetilde{K}^{\text{sig}} + \mu I_{N\times N}, \qquad \mu \in \mathbb{R}_{>0},
\]
so that for any subset \(X \subseteq \mathcal{D}\) with \(|X| = m\),
\(\Delta'(X) = \log \det\!\bigl(L_{X,X}\bigr)\) is (up to an additive constant)
the log-probability under the corresponding \(m\)-DPP.
The following algorithm samples a random subset \(X\) of size \(m\) from this \(m\)-DPP.

\begin{algorithm}[ht]
  \caption{Sampling an \(m\)-DPP on \(\mathcal{D}\)}
  \label{alg:k-dpp-sampling}
  \PaperBegin{algorithm3}{algorithm}{PaperTeal}{0pt}{0pt}
  \begin{algorithmic}[1]
    \Require Ground set \(\mathcal{D} = \{d_1,\dots,d_N\}\),
             similarity kernel \(k^{\text{sig}}\),
             regularizer \(\mu > 0\),
             target size \(m\).
    \State Form \(\widetilde{K}^{\text{sig}} \in \mathbb{R}^{N\times N}\) with
           \(\widetilde{K}^{\text{sig}}_{ij} = k^{\text{sig}}(d_i,d_j)\),
           and set \(L \gets \widetilde{K}^{\text{sig}} + \mu I_{N\times N}\).
    \State Compute eigendecomposition
           \(L = \sum_{n=1}^N \lambda_n v_n v_n^\top\),
           with eigenvalues \(\lambda_n \ge 0\) and orthonormal eigenvectors \(v_n \in \mathbb{R}^N\).

    \medskip
    \Statex \textbf{Phase 1: choose \(m\) eigenvectors.}
    \State Initialize elementary symmetric polynomials \(e_n^\ell\) for
           \(n=0,\dots,N\), \(\ell = 0,\dots,m\):
           \begin{align*}
             & e_0^0 \gets 1,\qquad
               e_n^0 \gets 1 \ \text{for } n = 1,\dots,N,\\
             & e_0^\ell \gets 0 \ \text{for } \ell = 1,\dots,m.
           \end{align*}
    \For{\(\ell = 1,\dots,m\)}
      \For{\(n = 1,\dots,N\)}
        \State \(e_n^\ell \gets e_{n-1}^\ell + \lambda_n e_{n-1}^{\ell-1}\)
      \EndFor
    \EndFor
    \State \(J \gets \emptyset,\quad \ell \gets m\).
    \For{\(n = N, N-1, \dots, 1\)}
      \If{\(\ell = 0\)} \State \textbf{break} \EndIf
      \State Draw \(u \sim \mathrm{Uniform}[0,1]\).
      \If{\(u \le \lambda_n e_{n-1}^{\ell-1} / e_n^\ell\)}
        \State \(J \gets J \cup \{n\}\), \quad \(\ell \gets \ell - 1\).
      \EndIf
    \EndFor
    \State Let \(V \gets \{v_n : n \in J\}\) be the selected eigenvectors.

    \medskip
    \Statex \textbf{Phase 2: sample items given \(V\).}
    \State \(Y \gets \emptyset\). \Comment{\(Y \subseteq \{1,\dots,N\}\) will index items in \(\mathcal{D}\)}
    \While{\(V \neq \emptyset\)}
      \State For each \(i \in \{1,\dots,N\}\), set
        \[
          \Pr(i) \propto \sum_{v \in V} (v_i)^2
        \]
        and sample \(i\) according to \(\Pr(i)\).
      \State \(Y \gets Y \cup \{i\}\).
      \State Update \(V\) to an orthonormal basis for the subspace of \(\mathrm{span}(V)\)
             orthogonal to the standard basis vector \(e_i\).
    \EndWhile
    \State \Return \(X = \{d_i : i \in Y\} \subseteq \mathcal{D}\) with \(|X| = m\).
  \end{algorithmic}
  \PaperEnd{algorithm3}
\end{algorithm}

\FloatBarrier
\appsection{Implementation Details} \label{app:exps_deatails}

\subsection{Robomimic Tasks}
For all Robomimic's tasks we use the official implementation provided by Robomimic~\cite{robomimic}. We now provide a brief description of the tasks, however we recommend the interested reader to consult \cite{robomimic} for the complete details. 
\subsubsection{Task Description}
\begin{itemize}
    \item \textbf{Can.} The robot must move a coke can from a large bin into a smaller target bin. This task is slightly more challenging than Lift, as grasping the can is more difficult than grasping the cube, and the robot must also accurately place the can into the bin. At the start of each episode, the can’s pose is randomized within the left bin, with its z-rotation sampled randomly.
    \item \textbf{Square.} The robot must pick up a square nut and place it onto a rod. This task is substantially more difficult than Lift and Pick Place Can because it requires precise manipulation to both grasp the nut and insert it onto the rod. At the start of each episode, the square nut’s pose is randomized within a square region on the table, including a randomly sampled z-rotation.
    \item \textbf{Transport.} Two robot arms must collaboratively transfer a hammer from a closed container on one shelf to a target bin on another shelf. One arm is responsible for retrieving the hammer from the container, while the other clears the target bin by moving a piece of trash into a nearby receptacle. Finally, one arm passes the hammer to the other, which then places it in the target bin. At the start of each episode, the positions of all bins, the lid, the trash cube, and the hammer are randomized within small square regions. Additionally, the z-rotations of the trash cube and the hammer are randomized over ranges of 108 and 60 degrees, respectively.
\end{itemize}

\begin{figure}[htbp]
  \setlength{\fboxsep}{0pt}
  \setlength{\fboxrule}{0.35pt}
  \centering
  \begin{subfigure}[t]{0.31\linewidth}
    \centering
    \PaperFrame{PaperTeal}{\includegraphics[width=\linewidth]{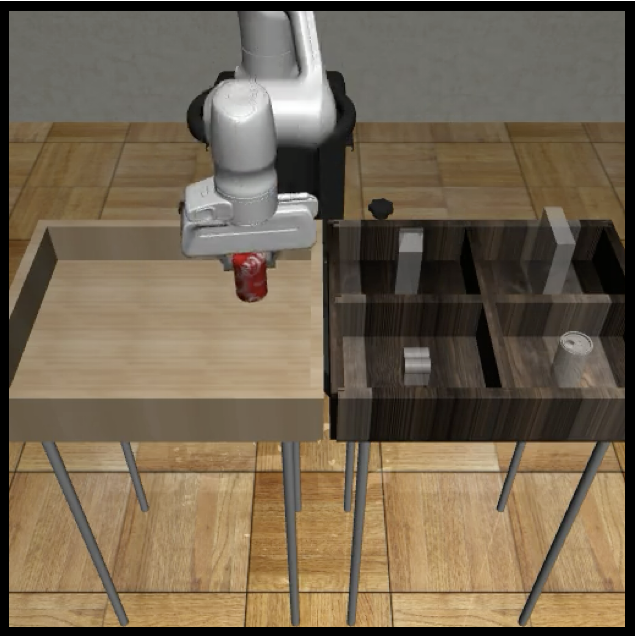}}
    \caption{Can Task}
  \end{subfigure}\hfill
  \begin{subfigure}[t]{0.31\linewidth}
    \centering
    \PaperFrame{PaperTeal}{\includegraphics[width=\linewidth]{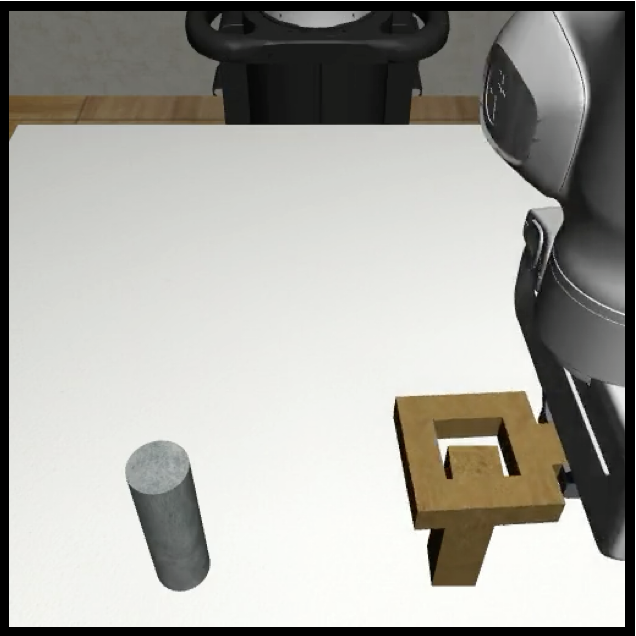}}
    \caption{Square Task}
  \end{subfigure}\hfill
  \begin{subfigure}[t]{0.31\linewidth}
    \centering
    \PaperFrame{PaperTeal}{\includegraphics[width=\linewidth]{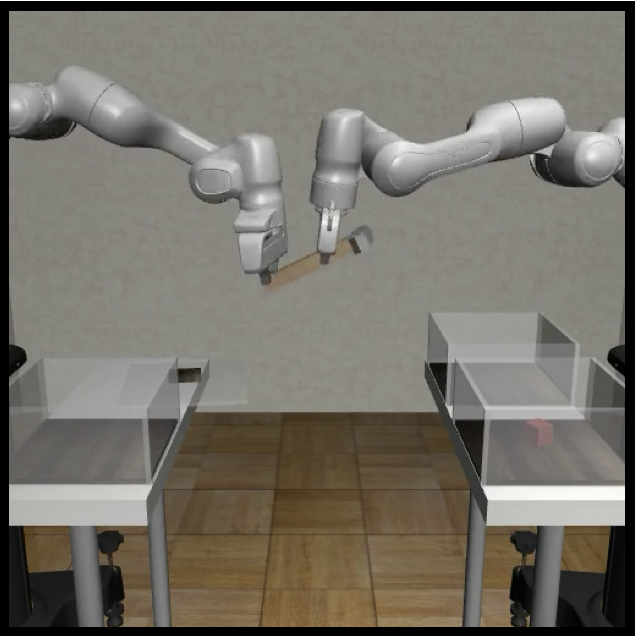}}
    \caption{Transport Task}
  \end{subfigure}
  \caption{Scene visualization of the 3 Robomimic tasks.}
  \label{fig:robomic_task_viz}
\end{figure}

\subsubsection{Dataset Description}
The PH datasets contain 200 demonstrations collected by a single, experienced teleoperator, whereas the MH datasets contain 300 demonstrations collected by 6 teleoperators of varying proficiency, each contributing 50 demonstrations. The 6 teleoperators are partitioned into a “better” group of 2 experienced operators, an “okay” group of 2 adequately skilled operators, and a “worse” group of 2 inexperienced operators.

In our experiments in Figure~\ref{fig:robomimic_results}, we used 20 demonstrations for validation and 180 for training when using the PH datasets, and 30 demonstrations for validation and 270 for training when using the MH datasets.

\subsubsection{Hyperparameters and model architecture}
The complete implementation details and configuration files are available in our official repository. The key details are summarized in Table~\ref{tab:robomimic_imp_deatails}.

\begingroup
\small
\setlength{\tabcolsep}{3.5pt}
\renewcommand{\arraystretch}{1.13}
\setlength{\LTcapwidth}{\linewidth}
\begin{longtable}{@{}>{\columncolor{white}[0pt][\tabcolsep]\raggedright\arraybackslash}p{.21\linewidth}
  >{\raggedright\arraybackslash}p{.29\linewidth}
  *{2}{>{\raggedright\arraybackslash}p{\dimexpr.166666\linewidth-8\tabcolsep/3\relax}}
  >{\columncolor{white}[\tabcolsep][0pt]\raggedright\arraybackslash}p{\dimexpr.166666\linewidth-8\tabcolsep/3\relax}@{}}
\caption{Hyperparameters and model settings for the Robomimic tasks.}
\label{tab:robomimic_imp_deatails}\\
\toprule
\PaperTableHead
\textbf{Category} & \textbf{Hyperparameter} & \textbf{can} & \textbf{square} & \textbf{transport} \\
\midrule
\endfirsthead
\multicolumn{5}{c}{\tablename\ \thetable{} -- continued from previous page} \\
\toprule
\PaperTableHead
\textbf{Category} & \textbf{Hyperparameter} & \textbf{can} & \textbf{square} & \textbf{transport} \\
\midrule
\endhead
\midrule
\multicolumn{5}{r}{Continued on next page} \\
\endfoot
\bottomrule
\endlastfoot

Policy & Policy type & RNN-based behavior cloning (BC-RNN) & Diffusion policy based on \citet{diffusion_policy} & Diffusion policy based on \citet{diffusion_policy} \\*
Policy & Batch size & 16 & 16 & 16 \\*
Policy & Number of steps & 100{,}000 & 100{,}000 & 200{,}000 \\
\midrule
Optimizer \& LR & Optimizer & Adam & AdamW & AdamW \\
Optimizer \& LR & Base learning rate & 1e-4 & 1e-4 & 1e-4 \\
Optimizer \& LR & Weight decay (L2) & 0.0 & 1e-6 & 1e-6 \\
Optimizer \& LR & LR scheduler type & multistep & cosine & cosine \\
Optimizer \& LR & Scheduler cycles & \textemdash & 0.5 & 0.5 \\
Optimizer \& LR & Scheduler epochs & none (constant LR, empty schedule) & \textemdash & \textemdash \\
Optimizer \& LR & Warmup steps & 0 (no warmup) & 500 & 500 \\
\midrule
Diffusion (DDPM) & UNet enabled & \textemdash & true & true \\
Diffusion (DDPM) & UNet down dims & \textemdash & {[}256, 512, 1024{]} & {[}256, 512, 1024{]} \\
Diffusion (DDPM) & UNet kernel size & \textemdash & 5 & 5 \\
Diffusion (DDPM) & UNet group norm groups & \textemdash & 8 & 8 \\
Diffusion (DDPM) & Diffusion step embed dim & \textemdash & 256 & 256 \\
Diffusion (DDPM) & EMA enabled & \textemdash & true & true \\
Diffusion (DDPM) & EMA power & \textemdash & 0.75 & 0.75 \\
Diffusion (DDPM) & DDPM enabled & \textemdash & true & true \\
Diffusion (DDPM) & DDPM train timesteps & \textemdash & 100 & 100 \\
Diffusion (DDPM) & DDPM inference timesteps & \textemdash & 100 & 100 \\
Diffusion (DDPM) & Beta schedule & \textemdash & squaredcos\_\allowbreak cap\_\allowbreak v2 & squaredcos\_\allowbreak cap\_\allowbreak v2 \\
Diffusion (DDPM) & Prediction type & \textemdash & epsilon (noise prediction) & epsilon (noise prediction) \\
Diffusion (DDPM) & Clip sample & \textemdash & true & true \\
Diffusion (DDPM) & DDIM enabled & \textemdash & false & false \\
\midrule
RNN Policy (BC-RNN) & RNN enabled & true & \textemdash & \textemdash \\
RNN Policy (BC-RNN) & RNN type & LSTM & \textemdash & \textemdash \\
RNN Policy (BC-RNN) & RNN hidden dim & 1000 & \textemdash & \textemdash \\
RNN Policy (BC-RNN) & RNN num layers & 2 & \textemdash & \textemdash \\
RNN Policy (BC-RNN) & RNN horizon & 10 & \textemdash & \textemdash \\
RNN Policy (BC-RNN) & RNN bidirectional & false & \textemdash & \textemdash \\
RNN Policy (BC-RNN) & Output distribution & GMM (mixture of Gaussians) & \textemdash & \textemdash \\
RNN Policy (BC-RNN) & GMM enabled & true & \textemdash & \textemdash \\
RNN Policy (BC-RNN) & GMM num modes & 5 & \textemdash & \textemdash \\
RNN Policy (BC-RNN) & GMM min std & 1e-4 & \textemdash & \textemdash \\
RNN Policy (BC-RNN) & GMM std activation & softplus & \textemdash & \textemdash \\
RNN Policy (BC-RNN) & GMM low-noise eval & true & \textemdash & \textemdash \\
RNN Policy (BC-RNN) & Gaussian head & disabled & \textemdash & \textemdash \\
RNN Policy (BC-RNN) & VAE & disabled & \textemdash & \textemdash \\
RNN Policy (BC-RNN) & Transformer & disabled & \textemdash & \textemdash \\
\midrule
Rollout / Evaluation & Rollout episodes & 50 & 50 & 50 \\
Rollout / Evaluation & Rollout horizon & 400 & 400 & 1100 \\
Rollout / Evaluation & Rollout and checkpoint rate & 5000 steps & 5000 steps & 10000 steps \\

\end{longtable}
\endgroup

\subsection{Metaworld Tasks}
For all Metaworld~\citep{metaworld, metaworld+} tasks, use the implementation provided by \citet{lerobot}. We now
provide a brief description of the tasks, however we recommend the interested reader to consult \cite{metaworld, metaworld+} for the complete details.

\subsubsection{Task Description}
\begin{itemize}
    \item \textbf{Door open} (door-open-v3). The robot must open a door with a revolving joint. The door positions are randomized.
    \item \textbf{Stick push} (stick-push-v3). The robot must grasp a stick and push a box using the stick. The stick positions are randomized.
    \item \textbf{Shelf place} (shelf-place-v3). The robot must pick and place a puck onto a shelf. The puck and shelf positions are randomized.
\end{itemize}

\begin{figure}[htbp]
  \setlength{\fboxsep}{0pt}
  \setlength{\fboxrule}{0.35pt}
  \centering

  \begin{subfigure}[t]{0.31\linewidth}
    \centering
    \PaperFrame{PaperTeal}{\includegraphics[width=\linewidth]{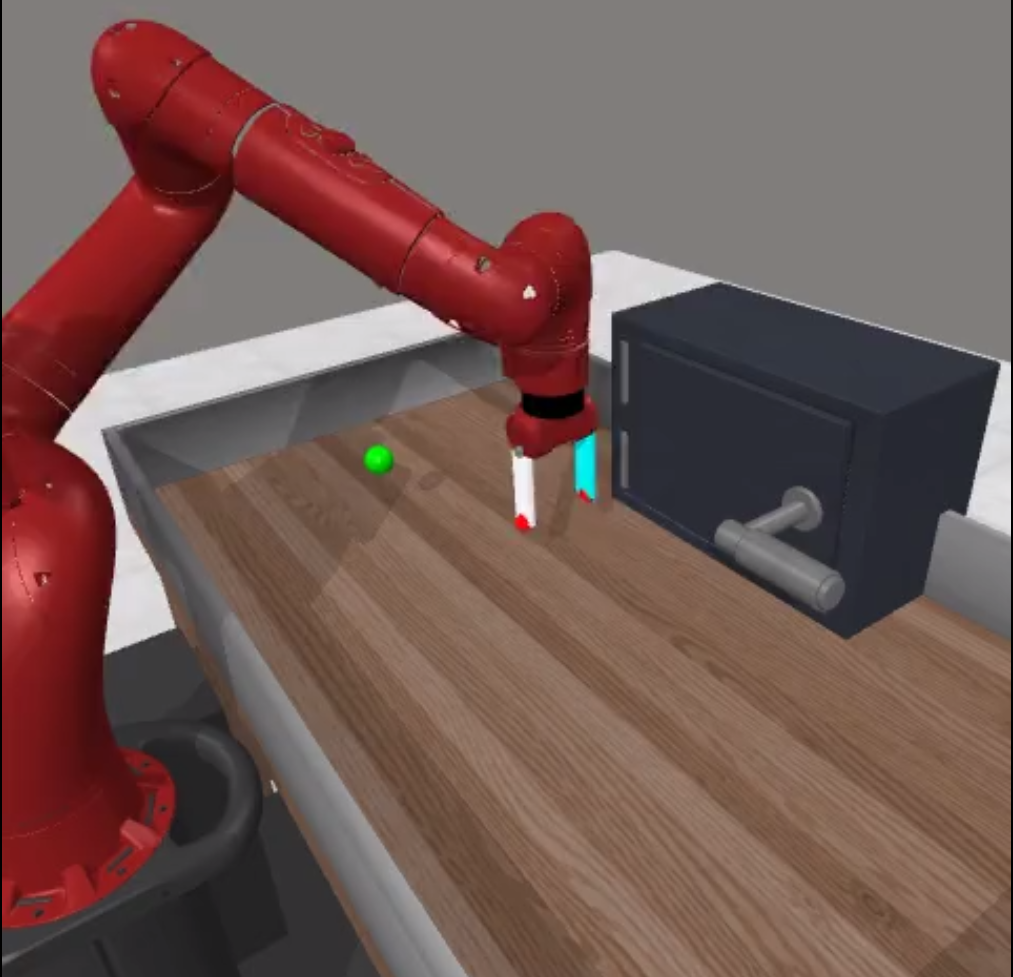}}
    \caption{Door open}
  \end{subfigure}\hfill
  \begin{subfigure}[t]{0.31\linewidth}
    \centering
    \PaperFrame{PaperTeal}{\includegraphics[width=\linewidth]{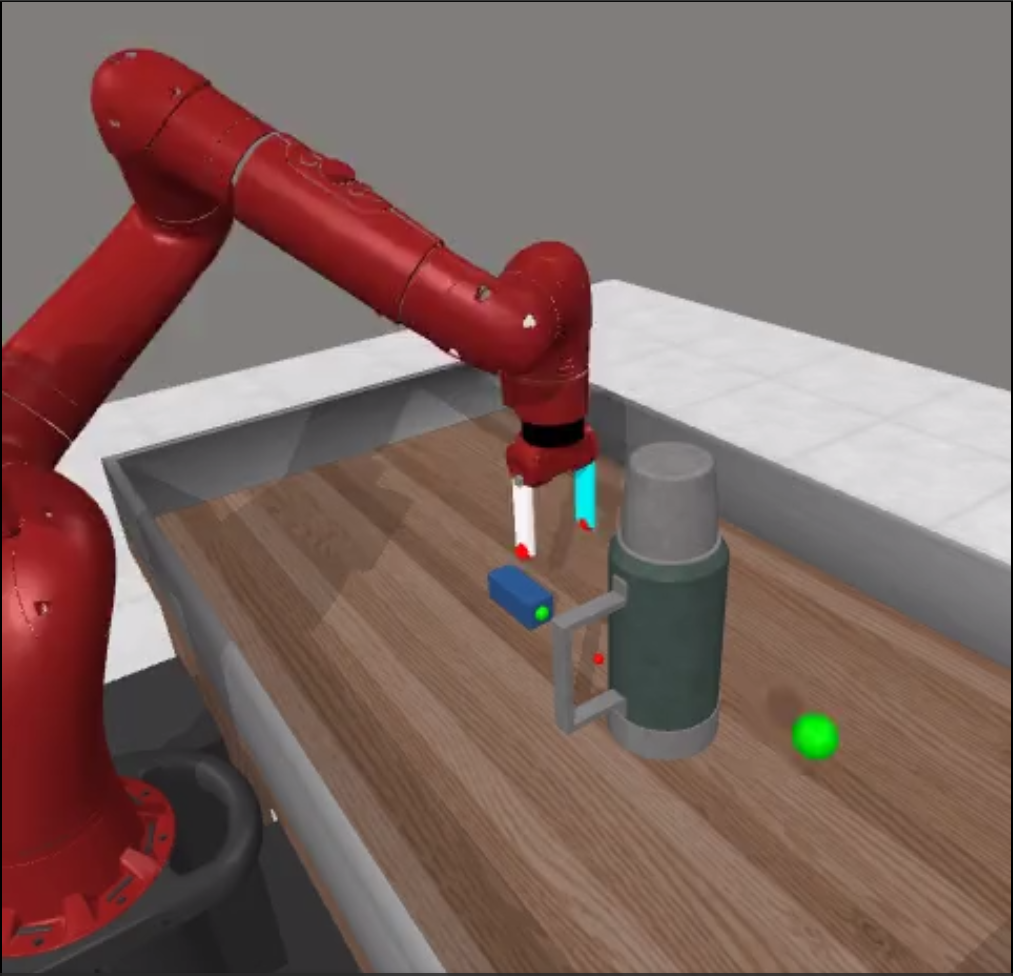}}
    \caption{Stick push}
  \end{subfigure}\hfill
  \begin{subfigure}[t]{0.31\linewidth}
    \centering
    \PaperFrame{PaperTeal}{\includegraphics[width=\linewidth]{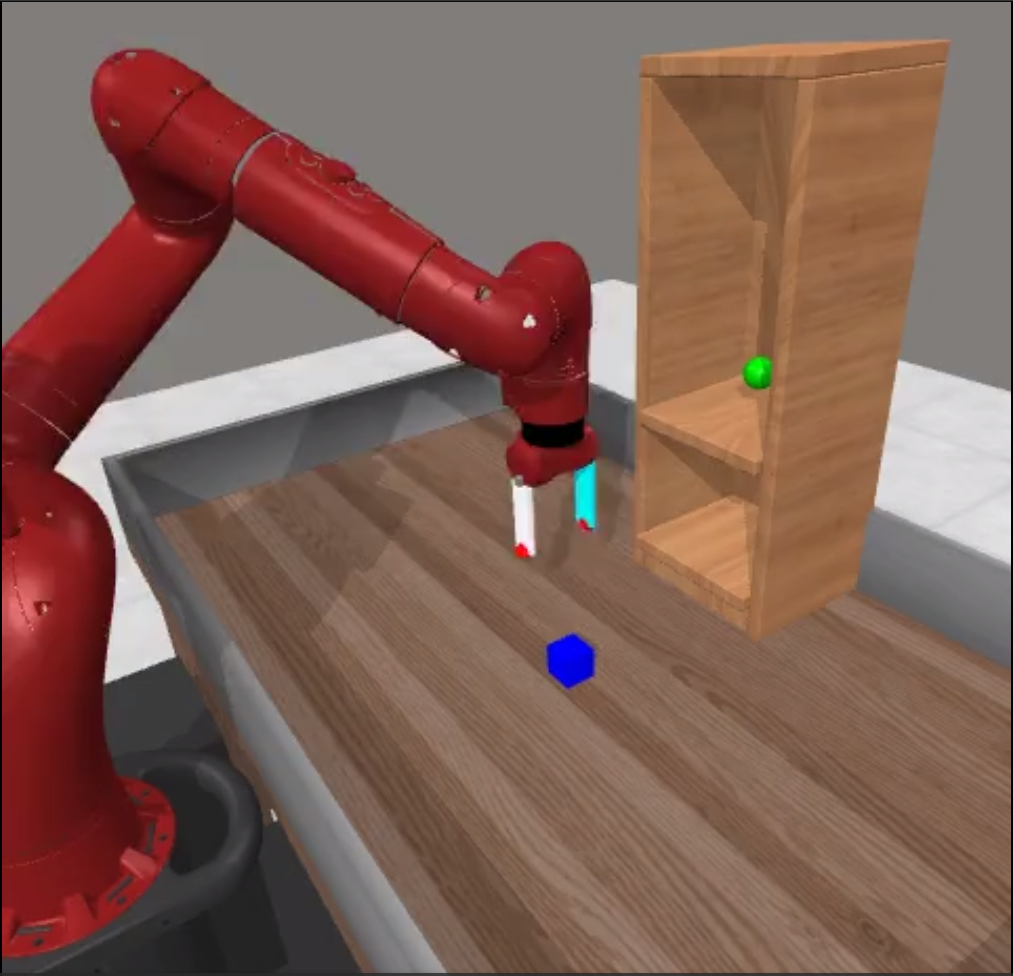}}
    \caption{Shelf place}
  \end{subfigure}

  \caption{Scene visualization of the 3 Metaworld tasks.}
  \label{fig:metaworld_task_viz}
\end{figure}

All goal positions are randomized as well.

\subsubsection{Dataset Description}
The datasets contain 50 demonstrations; however \citet{metaworld, metaworld+} do not disclose how the data was collected.

\subsubsection{Hyperparameters and model architecture}
The complete implementation details and configuration files are available in our official repository. The key details are summarized in Table~\ref{tab:metaworld_imp_deatails}. All tasks use the same configuration.

\FloatBarrier
\begingroup
\small
\setlength{\tabcolsep}{4pt}
\renewcommand{\arraystretch}{1.15}
\setlength{\LTcapwidth}{\linewidth}
\begin{longtable}{@{}>{\columncolor{white}[0pt][\tabcolsep]\raggedright\arraybackslash}p{.20\linewidth}
  >{\raggedright\arraybackslash}p{.28\linewidth}
  >{\columncolor{white}[\tabcolsep][0pt]\raggedright\arraybackslash}p{\dimexpr.52\linewidth-4\tabcolsep\relax}@{}}
\caption{SmolVLA hyperparameters and training configuration for the Metworld tasks.}
\label{tab:metaworld_imp_deatails}\\
\toprule
\PaperTableHead
\textbf{Category} & \textbf{Hyperparameter} & \textbf{Value} \\
\midrule
\endfirsthead
\multicolumn{3}{c}{\tablename\ \thetable{} -- continued from previous page} \\
\toprule
\PaperTableHead
\textbf{Category} & \textbf{Hyperparameter} & \textbf{Value} \\
\midrule
\endhead
\midrule
\multicolumn{3}{r}{Continued on next page} \\
\endfoot
\bottomrule
\endlastfoot

Model + Architecture 
& Model type & smolvla~\citep{smolvla} \\*
& Vision--language backbone & \texttt{HuggingFaceTB/\allowbreak SmolVLM2-\allowbreak 500M-\allowbreak Video-\allowbreak Instruct} \\*
& VLM transformer layers & $16$ transformer layers from the VLM stack \\*
& Number of expert layers & $0$ (expert interleaved into VLM stack; only expert is trained) \\*
& Expert width multiplier & $0.75$ (expert MLP width vs.\ VLM width) \\*
& Attention mode & cross-attn w/ periodic self-attn \\*
& Self-attention frequency & Self-attention every $2$ layers \\*
& KV cache for inference & Enabled  \\
\midrule\noalign{\penalty0}
Training / Finetuning behaviour
& Vision encoder frozen & true \\*
& Train expert head only & true (only expert + small heads get gradients) \\*
& Train state projection MLP & true \\*
& Load pretrained VLM weights & true (loads pretrained SmolVLM2 weights) \\*
& Image special tokens & false \\
\midrule\noalign{\penalty0}
Optimizer \& LR schedule
& Optimizer & AdamW (from LeRobot training utilities) \\*
& Base learning rate & $1\times10^{-4}$ peak learning rate \\*
& Adam betas & $(0.9,\ 0.95)$ \\*
& Adam epsilon & $1\times10^{-8}$ \\*
& Weight decay & $1\times10^{-10}$ (very light weight decay) \\*
& Gradient clipping norm & $10$ (global grad clipping) \\*
& LR scheduler type & Cosine decay with warmup \\*
& Warmup steps & $1000$ warmup steps \\*
& Decay steps & $30000$ cosine decay steps \\*
& Minimum learning rate & $2.5\times10^{-6}$ (minimum LR plateau) \\
\midrule\noalign{\penalty0}
Overrides from training command
& Batch size & 8 \\*
& Number of training steps & 40000 \\
\midrule\noalign{\penalty0}
Rollout / Evaluation
& Rollout episodes & 50 \\*
& Rollout horizon & 400 \\*
& Rollout and checkpoint rate & 2000 steps \\

\end{longtable}
\endgroup
\FloatBarrier

\subsection{Real Tasks}

All real tasks we use the official lerobot library~\citep{lerobot}. We now provide a brief description of the tasks.

\subsubsection{Task Description}
\begin{itemize}
    \item \textbf{Drawer open.} The robot must open a drawer by pulling a handle. The position of the drawer is fixed for all episodes; however, 5 different lighting conditions were used.
    \item \textbf{Pick and place marker.} The robot must pick up a marker and then drop into a pen holder. The position of the pen holder and marker are independently sampled from a fixed set of 7 locations (2 novel location were used that were not preset during the recording of the dataset).
    \item \textbf{Organize objects.} If a marker/pen is placed on the table, the robot must place it on a
    book, else if a charging brick is placed on the table, the robot must (1) first open a drawer, (2) then place
    the brick in the drawer, (3) and finally then close the drawer. The position of the drawer is fixed for all episodes. The position of the charging brick, marker and book are sampled from a fixed set of 5 random positions. There are no episodes where both the marker and the charging brick are placed on the table.
    \item \textbf{Mug Drag.} The robot must grasp a hook and drag a mug to a square goal region by using the hook as an intermediate tool. The position of the hook is randomized and the position of the mug is always to the left of the goal, however the distance to the goal is randomized.
\end{itemize}

\begin{figure}[htbp]
  \setlength{\fboxsep}{0pt}
  \setlength{\fboxrule}{0.35pt}
  \centering
  \begin{subfigure}[t]{0.31\linewidth}
    \centering
    \PaperFrame{PaperTeal}{\includegraphics[width=\linewidth]{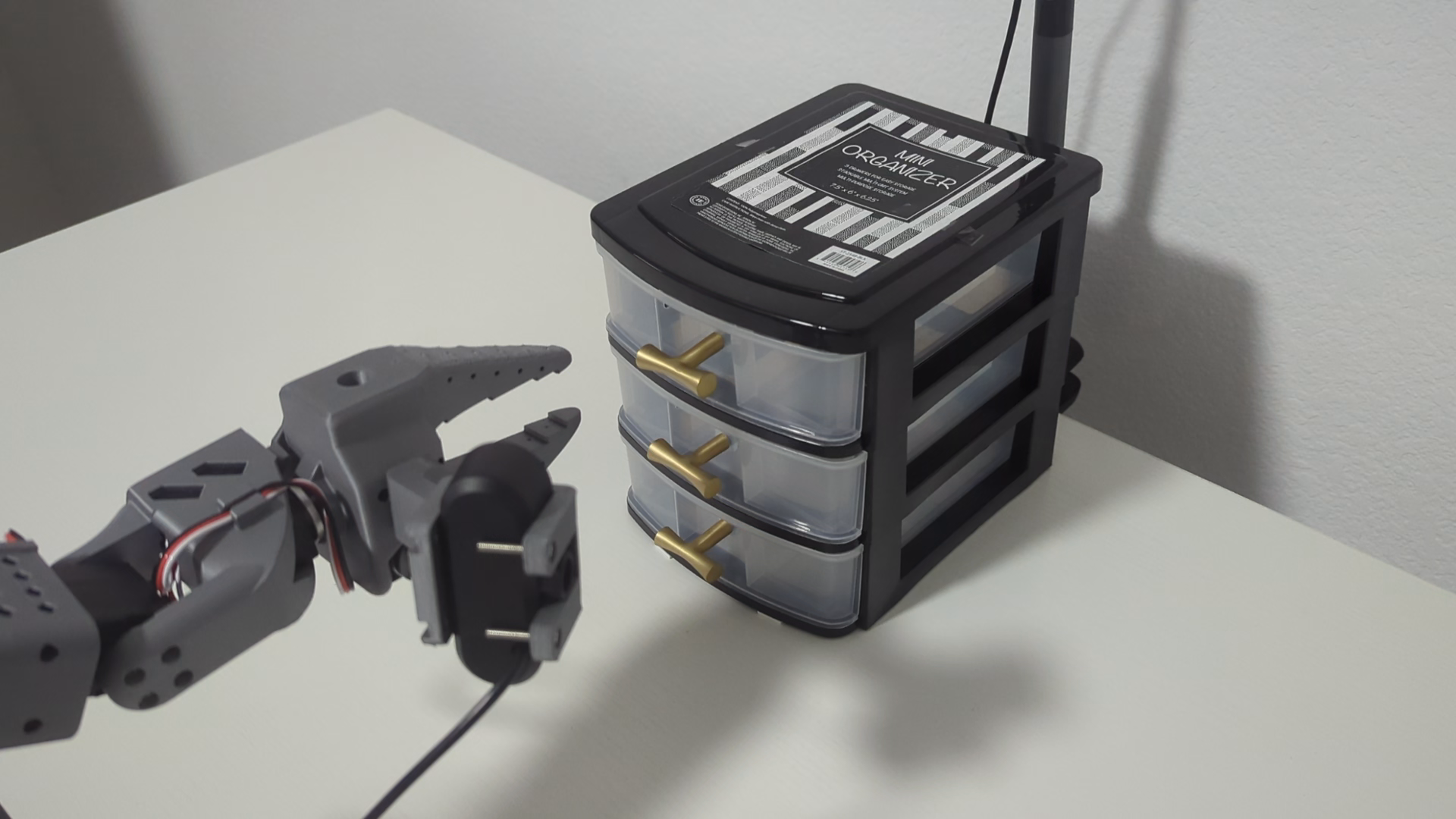}}
    \caption{Drawer open}
  \end{subfigure}\hfill
  \begin{subfigure}[t]{0.31\linewidth}
    \centering
    \PaperFrame{PaperTeal}{\includegraphics[width=\linewidth]{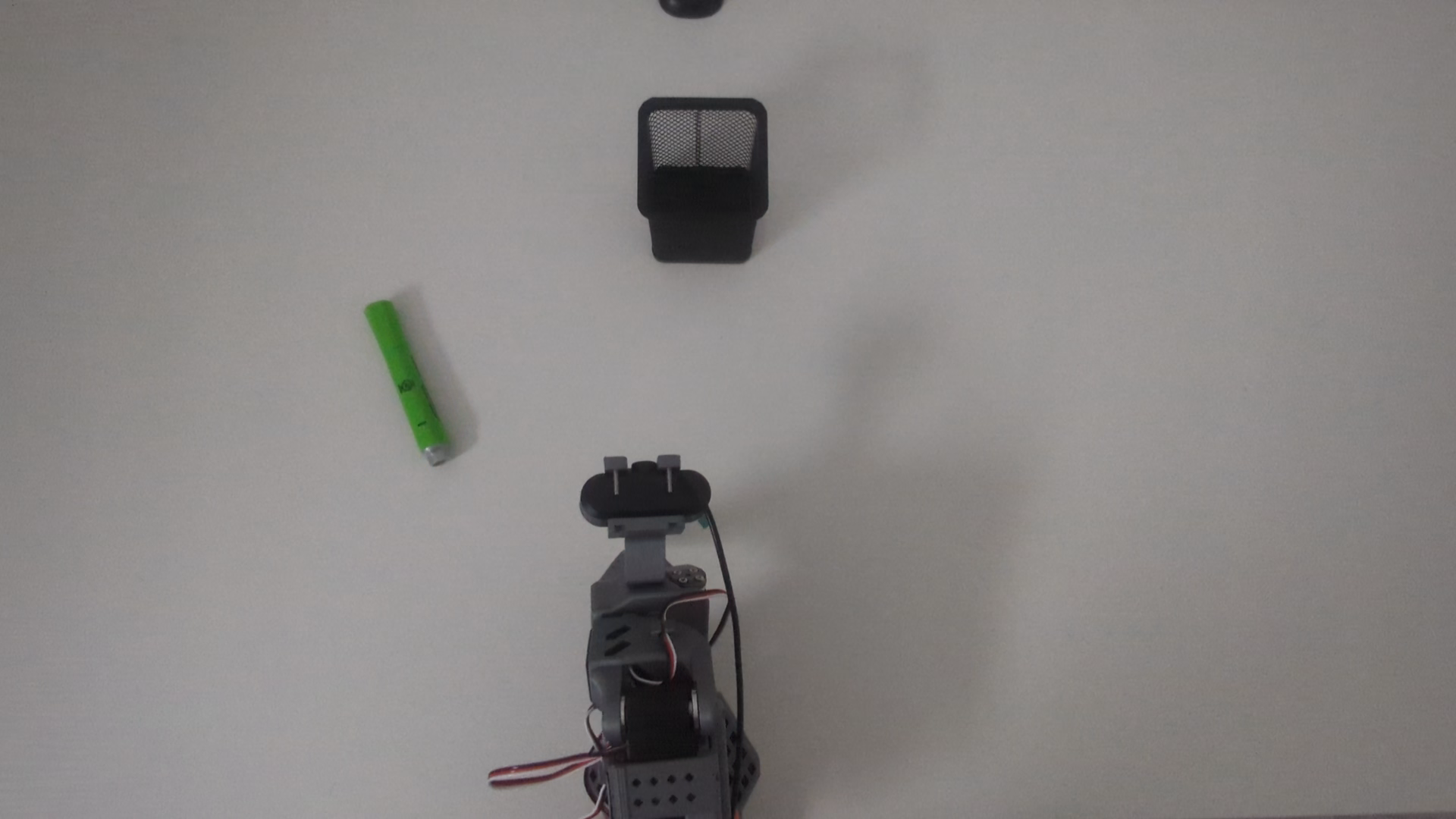}}
    \caption{Pick and place marker}
  \end{subfigure}\hfill
  \begin{subfigure}[t]{0.31\linewidth}
    \centering
    \PaperFrame{PaperTeal}{\includegraphics[width=\linewidth]{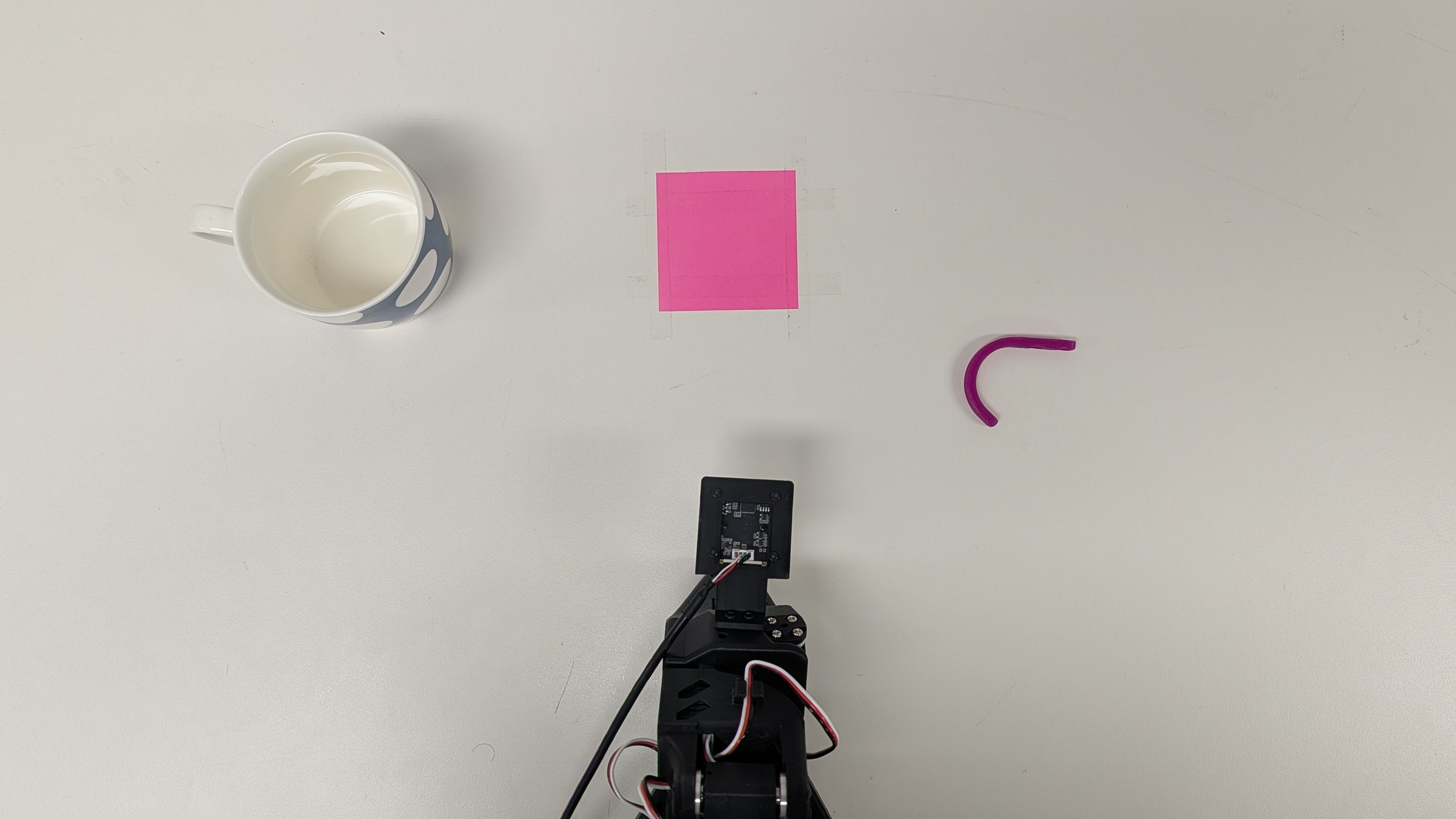}}
    \caption{Mug Drag}
  \end{subfigure}
  \caption{Scene visualization of the drawer open task, pick and place marker task and mug drag task.}
  \label{fig:real_task_viz}
\end{figure}

\begin{figure}[htbp]
  \setlength{\fboxsep}{0pt}
  \setlength{\fboxrule}{0.35pt}
  \centering
  \begin{subfigure}[t]{0.31\linewidth}
    \centering
    \PaperFrame{PaperTeal}{\includegraphics[width=\linewidth]{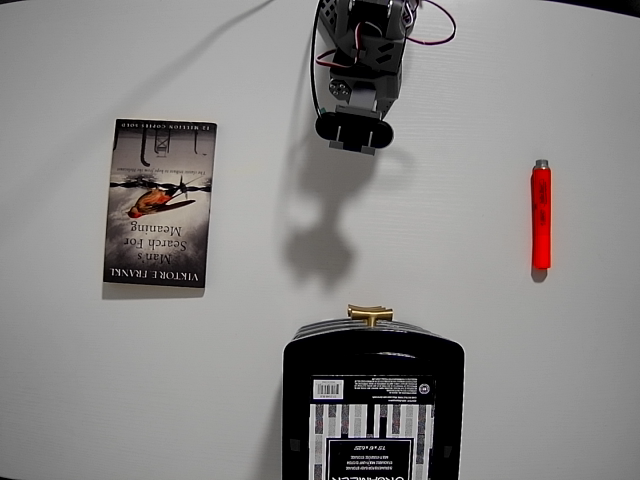}}
    \caption{Initial positions of objects in an episode with a marker.}
  \end{subfigure}\hspace{0.035\linewidth}
  \begin{subfigure}[t]{0.31\linewidth}
    \centering
    \PaperFrame{PaperTeal}{\includegraphics[width=\linewidth]{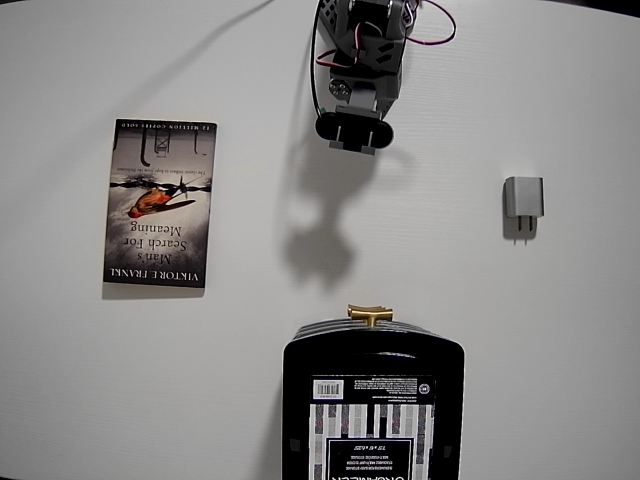}}
    \caption{Initial positions of objects in an episode with a charging brick.}
  \end{subfigure}
  \caption{Scene visualization of the organize objects tasks.}
  \label{fig:organize_objects_task_viz}
\end{figure}

\subsubsection{Dataset Description}
The drawer open and Pick and place marker datasets contain 50 demonstrations, whereas the organize objects and mug drag datasets contains 100 demonstrations collected by a single experienced teleoperator. The pick and place marker and mug drag tasks use two cameras: one mounted on the robot’s wrist and another mounted above the workspace providing a top-down view. The drawer opening and object organization tasks use an additional camera with a diagonal side view. All cameras capture plain RGB images at 30 FPS with a resolution of 480 (height) by 640 (width) pixels.

\subsubsection{Hyperparameters and model architecture}
We use the same configuration and setup as in Table~\ref{tab:metaworld_imp_deatails}, with the only change being batch size set to 64, training steps set to 60000 steps and rollout rate every 20000 steps. 
\appsection{Additional Results} \label{app:add_results}

\subsection{Additional Results for Section~\ref{sec:ablations}}
Figure~\ref{fig:siglevels_results} shows the effect of the number of signature levels on the success rate for the Can task.

\FloatBarrier
\begin{figure}[H]
    \centering
    \PaperFrame{PaperTeal}{\includegraphics[width=1.0\linewidth, keepaspectratio]{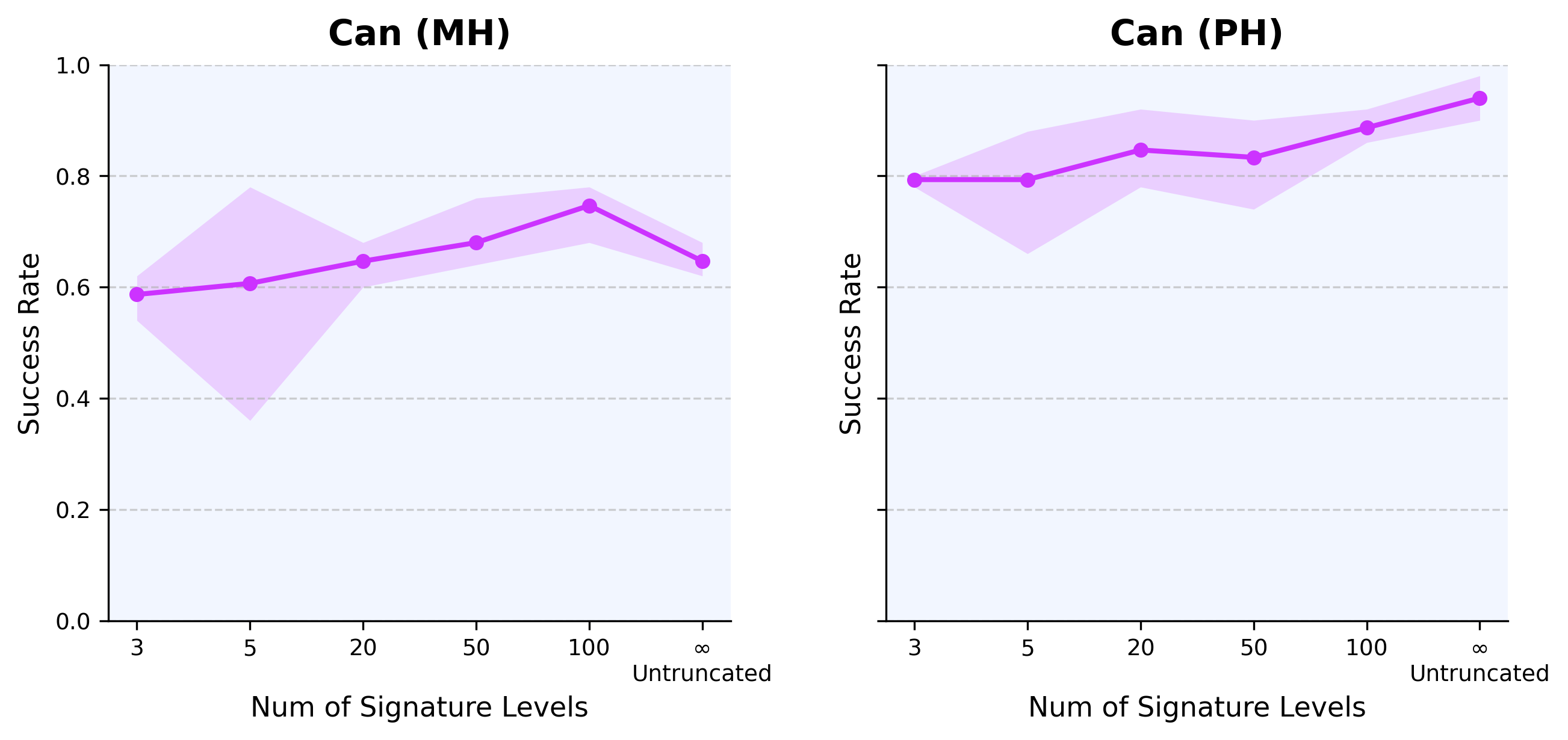}}
    \caption{Number of signature levels vs.\ success rate results. Results are averaged over three random seeds, and the error bars indicate the minimum and maximum values. Parameters (such as dataset size) are set to \(m = 120\) and \(p = 0\).}
    \label{fig:siglevels_results}
\end{figure}
\FloatBarrier

\subsection{Additional Results for Section~\ref{sec:results}}
Table~\ref{tab:full_results_metaworld} provides the complete and full results for metaworld experiments that were summarized in Figure~\ref{fig:metaworld_results}.

\FloatBarrier

\begingroup
\setlength{\tabcolsep}{5.5pt}
\renewcommand{\arraystretch}{1.12}
\setlength{\LTcapwidth}{\linewidth}
\begin{longtable}{@{}>{\columncolor{white}[0pt][\tabcolsep]}l l c c c c c c >{\columncolor{white}[\tabcolsep][0pt]}c@{}}
\caption{Success rates for all metaworld manipulation tasks over 3 seeds. Mean, minimum, and maximum are computed across seeds.}
\label{tab:full_results_metaworld}\\

\toprule
\PaperTableHead
Task & Method & \makecell{Num. of\\Demos.} & Seed 1 & Seed 2 & Seed 3 & Mean & Min & Max \\
\midrule
\endfirsthead

\multicolumn{9}{c}{\tablename\ \thetable{} -- continued from previous page} \\
\toprule
\PaperTableHead
Task & Method & \makecell{Num. of\\Demos.} & Seed 1 & Seed 2 & Seed 3 & Mean & Min & Max \\
\midrule
\endhead

\midrule
\multicolumn{9}{r}{Continued on next page} \\
\endfoot

\bottomrule
\endlastfoot

\PaperMethodRow
Door open   & FAKTUAL & 25 & 0.72 & 0.60 & 0.74 & 0.69 & 0.60 & 0.74 \\*
            & random  & 25 & 0.62 & 0.52 & 0.66 & 0.60 & 0.52 & 0.66 \\*
\PaperMethodRow
            & FAKTUAL & 35 & 0.66 & 0.62 & 0.74 & 0.67 & 0.62 & 0.74 \\*
            & random  & 35 & 0.70 & 0.56 & 0.78 & 0.68 & 0.56 & 0.78 \\*
\PaperMethodRow
            & FAKTUAL & 45 & 0.68 & 0.64 & 0.68 & 0.67 & 0.64 & 0.68 \\*
            & random  & 45 & 0.72 & 0.50 & 0.86 & 0.69 & 0.50 & 0.86 \\*
            & full dataset & 50 & 0.68 & 0.56 & 0.76 & 0.67 & 0.56 & 0.76 \\

\midrule
\PaperMethodRow
Shelf place & FAKTUAL & 25 & 0.12 & 0.08 & 0.08 & 0.09 & 0.08 & 0.12 \\*
            & random  & 25 & 0.08 & 0.04 & 0.08 & 0.07 & 0.04 & 0.08 \\*
\PaperMethodRow
            & FAKTUAL & 35 & 0.12 & 0.08 & 0.10 & 0.10 & 0.08 & 0.12 \\*
            & random  & 35 & 0.10 & 0.08 & 0.08 & 0.09 & 0.08 & 0.10 \\*
\PaperMethodRow
            & FAKTUAL & 45 & 0.12 & 0.06 & 0.08 & 0.09 & 0.06 & 0.12 \\*
            & random  & 45 & 0.10 & 0.04 & 0.10 & 0.08 & 0.04 & 0.10 \\*
            & full dataset & 50 & 0.06 & 0.04 & 0.08 & 0.06 & 0.04 & 0.08 \\

\midrule
\PaperMethodRow
Stick push  & FAKTUAL & 25 & 0.44 & 0.26 & 0.22 & 0.29 & 0.18 & 0.44 \\*
            & random  & 25 & 0.42 & 0.16 & 0.18 & 0.27 & 0.16 & 0.42 \\*
\PaperMethodRow
            & FAKTUAL & 35 & 0.46 & 0.26 & 0.20 & 0.30 & 0.18 & 0.46 \\*
            & random  & 35 & 0.42 & 0.16 & 0.20 & 0.27 & 0.16 & 0.42 \\*
\PaperMethodRow
            & FAKTUAL & 45 & 0.46 & 0.20 & 0.18 & 0.28 & 0.18 & 0.46 \\*
            & random  & 45 & 0.42 & 0.18 & 0.22 & 0.27 & 0.18 & 0.42 \\*
            & full dataset & 50 & 0.46 & 0.16 & 0.16 & 0.26 & 0.16 & 0.46 \\

\end{longtable}
\endgroup

\FloatBarrier

Table~\ref{tab:full_robomimic_results} provides the complete and full results for robomimic experiments that were summarized in Figure~\ref{fig:robomimic_results}.

\FloatBarrier

\begingroup
\setlength{\tabcolsep}{5.5pt}
\renewcommand{\arraystretch}{1.12}
\setlength{\LTcapwidth}{\linewidth}
\begin{longtable}{@{}>{\columncolor{white}[0pt][\tabcolsep]}l l c c c c c c >{\columncolor{white}[\tabcolsep][0pt]}c@{}}
\caption{Success rates for all robomimic manipulation tasks over 3 seeds. Mean, minimum, and maximum are computed across seeds.}
\label{tab:full_robomimic_results}\\

\toprule
\PaperTableHead
Task & Method & \makecell{Num.\ of\\Demos.} & Seed 1 & Seed 2 & Seed 3 & Mean & Min & Max \\
\midrule
\endfirsthead

\multicolumn{9}{c}{\tablename\ \thetable{} -- continued from previous page} \\
\toprule
\PaperTableHead
Task & Method & \makecell{Num.\ of\\Demos.} & Seed 1 & Seed 2 & Seed 3 & Mean & Min & Max \\
\midrule
\endhead

\midrule
\multicolumn{9}{r}{Continued on next page} \\
\endfoot

\bottomrule
\endlastfoot

\PaperMethodRow
Can (MH) & FAKTUAL & 60 & 0.72 & 0.80 & 0.70 & 0.74 & 0.70 & 0.80 \\*
 & random & 60 & 0.62 & 0.72 & 0.76 & 0.70 & 0.62 & 0.76 \\*
\PaperMethodRow
 & FAKTUAL & 120 & 0.82 & 0.88 & 0.68 & 0.79 & 0.68 & 0.88 \\*
 & random & 120 & 0.76 & 0.70 & 0.78 & 0.75 & 0.70 & 0.78 \\*
\PaperMethodRow
 & FAKTUAL & 180 & 0.76 & 0.80 & 0.74 & 0.77 & 0.74 & 0.80 \\*
 & random & 180 & 0.70 & 0.74 & 0.78 & 0.74 & 0.70 & 0.78 \\*
\PaperMethodRow
 & FAKTUAL & 240 & 0.82 & 0.86 & 0.80 & 0.83 & 0.80 & 0.86 \\*
 & random & 240 & 0.78 & 0.52 & 0.86 & 0.72 & 0.52 & 0.86 \\*
 & full dataset & 280 & 0.80 & 0.74 & 0.84 & 0.79 & 0.74 & 0.84 \\

\midrule
\PaperMethodRow
Can (PH) & FAKTUAL & 30 & 0.76 & 0.82 & 0.70 & 0.76 & 0.70 & 0.82 \\*
 & random & 30 & 0.58 & 0.66 & 0.66 & 0.63 & 0.58 & 0.66 \\*
\PaperMethodRow
 & FAKTUAL & 60 & 0.86 & 0.90 & 0.76 & 0.84 & 0.76 & 0.90 \\*
 & random & 60 & 0.76 & 0.78 & 0.70 & 0.75 & 0.70 & 0.78 \\*
\PaperMethodRow
 & FAKTUAL & 90 & 0.84 & 0.92 & 0.86 & 0.87 & 0.84 & 0.92 \\*
 & random & 90 & 0.90 & 0.82 & 0.80 & 0.84 & 0.80 & 0.90 \\*
\PaperMethodRow
 & FAKTUAL & 120 & 0.98 & 0.98 & 0.90 & 0.95 & 0.90 & 0.98 \\*
 & random & 120 & 0.88 & 0.82 & 0.88 & 0.86 & 0.82 & 0.88 \\*
\PaperMethodRow
 & FAKTUAL & 150 & 0.94 & 0.88 & 0.94 & 0.92 & 0.88 & 0.94 \\*
 & random & 150 & 0.80 & 0.90 & 0.72 & 0.81 & 0.72 & 0.90 \\*
 & full dataset & 180 & 0.86 & 0.88 & 0.96 & 0.90 & 0.86 & 0.96 \\

\midrule
\PaperMethodRow
Square (MH) & FAKTUAL & 60 & 0.28 & 0.26 & 0.30 & 0.28 & 0.26 & 0.30 \\*
 & random & 60 & 0.30 & 0.26 & 0.32 & 0.29 & 0.26 & 0.32 \\*
\PaperMethodRow
 & FAKTUAL & 120 & 0.46 & 0.50 & 0.44 & 0.47 & 0.44 & 0.50 \\*
 & random & 120 & 0.44 & 0.48 & 0.44 & 0.45 & 0.44 & 0.48 \\*
\PaperMethodRow
 & FAKTUAL & 180 & 0.56 & 0.60 & 0.54 & 0.57 & 0.54 & 0.60 \\*
 & random & 180 & 0.54 & 0.52 & 0.54 & 0.53 & 0.52 & 0.54 \\*
\PaperMethodRow
 & FAKTUAL & 240 & 0.62 & 0.64 & 0.66 & 0.64 & 0.62 & 0.66 \\*
 & random & 240 & 0.62 & 0.56 & 0.60 & 0.59 & 0.56 & 0.62 \\*
 & full dataset & 280 & 0.60 & 0.56 & 0.62 & 0.59 & 0.56 & 0.62 \\

\midrule
\PaperMethodRow
Square (PH) & FAKTUAL & 30 & 0.20 & 0.24 & 0.20 & 0.21 & 0.20 & 0.24 \\*
 & random & 30 & 0.22 & 0.22 & 0.30 & 0.25 & 0.22 & 0.30 \\*
\PaperMethodRow
 & FAKTUAL & 60 & 0.48 & 0.48 & 0.46 & 0.47 & 0.46 & 0.48 \\*
 & random & 60 & 0.48 & 0.58 & 0.58 & 0.55 & 0.48 & 0.58 \\*
\PaperMethodRow
 & FAKTUAL & 90 & 0.64 & 0.66 & 0.70 & 0.67 & 0.64 & 0.70 \\*
 & random & 90 & 0.64 & 0.64 & 0.54 & 0.61 & 0.54 & 0.64 \\*
\PaperMethodRow
 & FAKTUAL & 120 & 0.76 & 0.76 & 0.72 & 0.75 & 0.72 & 0.76 \\*
 & random & 120 & 0.74 & 0.72 & 0.72 & 0.73 & 0.72 & 0.74 \\*
\PaperMethodRow
 & FAKTUAL & 150 & 0.78 & 0.78 & 0.80 & 0.79 & 0.78 & 0.80 \\*
 & random & 150 & 0.76 & 0.74 & 0.76 & 0.75 & 0.74 & 0.76 \\*
 & full dataset & 180 & 0.78 & 0.74 & 0.78 & 0.77 & 0.74 & 0.78 \\

\midrule
\PaperMethodRow
Transport (MH) & FAKTUAL & 60 & 0.08 & 0.08 & 0.08 & 0.08 & 0.08 & 0.08 \\*
 & random & 60 & 0.14 & 0.04 & 0.08 & 0.09 & 0.04 & 0.14 \\*
\PaperMethodRow
 & FAKTUAL & 120 & 0.14 & 0.10 & 0.10 & 0.11 & 0.10 & 0.14 \\*
 & random & 120 & 0.12 & 0.10 & 0.06 & 0.09 & 0.06 & 0.12 \\*
\PaperMethodRow
 & FAKTUAL & 180 & 0.16 & 0.14 & 0.12 & 0.14 & 0.12 & 0.16 \\*
 & random & 180 & 0.16 & 0.16 & 0.16 & 0.16 & 0.16 & 0.16 \\*
\PaperMethodRow
 & FAKTUAL & 240 & 0.26 & 0.14 & 0.20 & 0.20 & 0.14 & 0.26 \\*
 & random & 240 & 0.22 & 0.08 & 0.18 & 0.16 & 0.08 & 0.22 \\*
 & full dataset & 280 & 0.20 & 0.14 & 0.16 & 0.17 & 0.14 & 0.20 \\

\end{longtable}
\endgroup

\FloatBarrier



\appsection{Proofs} \label{app:proofs}
\begin{lemma}
    Consider the setting from Definitions~\ref{def:sig_shannon_entropy} \& \ref{def:sig_neumann_entropy}. Then 
    \[H^{sig}_\text{von neumann}(d_1, d_2, \cdots, d_n) = H_\text{shannon}^{sig}(d_1, d_2, \cdots, d_n).
    \]
\end{lemma}

\begin{proof}
    Consider the signature kernel matrix \(K^{sig}\) from Definition~\ref{def:sig_shannon_entropy}. Since the signature kernel matrix \(K^{sig}\) is positive semidefinite and diagonalizable, it has an eigendecomposition of the form \(H \,\operatorname{diag}(\lambda_1, \dots, \lambda_n)\, H^{-1}\). Then \(\ln K^{sig} = H \,\operatorname{diag}(\ln \lambda_1, \dots, \ln \lambda_n)\, H^{-1}\). Trivially, \(\text{Tr}\left( H \,\operatorname{diag}(\lambda_1, \dots, \lambda_n)\, H^{-1} \right) = \text{Tr}\left(\operatorname{diag}(\lambda_1, \dots, \lambda_n)\ \right)\) because the trace is similarity-invariant. Then we have,
    \begin{align*}
    \operatorname{Tr}\left(\frac{K^{\text{sig}}}{n}\ln\frac{K^{\text{sig}}}{n}\right)
      &= \operatorname{Tr}\Bigl(
           H \,\operatorname{diag}(\lambda_1, \dots, \lambda_n)\, H^{-1} \\
      &\qquad\qquad
           \ln\bigl(
             H \,\operatorname{diag}(\lambda_1, \dots, \lambda_n)\, H^{-1}
           \bigr)
         \Bigr) \\
      &= \operatorname{Tr}\Bigl(
           H \,\operatorname{diag}(\lambda_1, \dots, \lambda_n)\, H^{-1} \, \\
      &\qquad\qquad
           H\,\operatorname{diag}(\ln \lambda_1, \dots, \ln \lambda_n)\, H^{-1}
         \Bigr) \\
      &= \operatorname{Tr}\Bigl(
           \operatorname{diag}(\lambda_1, \dots, \lambda_n)\,
           \operatorname{diag}(\ln \lambda_1, \dots, \ln \lambda_n)
         \Bigr) \\
      &= \sum_{i=1}^n \lambda_i \ln \lambda_i.
    \end{align*}

    Hence,
    \begin{equation*}
    \begin{aligned}
        H^{sig}_\text{von neumann}(d_1, d_2, \cdots, d_n) &= - \text{Tr}\left(\frac{K^{sig}}{n}\ln\frac{K^{sig}}{n}\right)\\
        &=-\sum_{i=1}^n \lambda_i \ln \lambda_i \\
        &= H_\text{shannon}^{sig}(d_1, d_2, \cdots, d_n).
    \end{aligned}
    \end{equation*}
\end{proof}

\end{document}